\pdfoutput=1

\documentclass[11pt]{article}

\usepackage{acl}

\usepackage{times}
\usepackage{latexsym}
\usepackage{lipsum} 
\usepackage{multirow}
\usepackage{amsthm}

\usepackage[T1]{fontenc}

\usepackage[utf8]{inputenc}

\usepackage{microtype}

\usepackage{inconsolata}

\usepackage[many]{tcolorbox}
\definecolor{bgblue}{RGB}{245,243,253}
\definecolor{ttblue}{RGB}{91,194,224}
\usepackage{booktabs}
\usepackage{subcaption}
\usepackage{graphicx}
\usepackage{algorithm}
\usepackage{algpseudocode}
\usepackage{mathtools}
\usepackage{amsfonts}
\usepackage{dblfloatfix}
\usepackage{nameref}
\newtheorem{theorem}{Theorem}
\newtheorem{lemma}{Lemma}[theorem]
\newtheorem{assumption}{Assumption}[theorem]
\newtheorem{assumptionalt}{Assumption}[assumption]
\newenvironment{assumptionp}[1]{
\renewcommand\theassumptionalt{#1}
\assumptionalt
}{\endassumptionalt}

\definecolor{main}{HTML}{5989cf}    
\definecolor{sub}{HTML}{cde4ff}     

\usepackage{listings}
\title{TRAQ: Trustworthy Retrieval Augmented Question Answering \\ via Conformal Prediction}

\author{Shuo Li \\
  University of Pennsylvania\\
  \texttt{lishuo1@seas.upenn.edu} \\\And
  Sangdon Park \\
  Pohang University of Science and Technology \\
  \texttt{sangdon@postech.ac.kr} \\ 
  \AND
  Insup Lee \\
  University of Pennsylvania \\
  \texttt{lee@cis.upenn.edu} \\ \And
  Osbert Bastani \\
  University of Pennsylvania \\
  \texttt{obastani@seas.upenn.edu}
  }

\begin{document}

\maketitle
\begin{abstract}
When applied to open-domain question answering, large language models (LLMs) frequently generate incorrect responses based on made-up facts, which are called \textit{hallucinations}. Retrieval augmented generation (RAG) is a promising strategy to avoid hallucinations, but it does not provide guarantees on its correctness. To address this challenge, we propose the Trustworthy Retrieval Augmented Question Answering, or \textit{TRAQ}, which provides the first end-to-end statistical correctness guarantee for RAG. TRAQ uses conformal prediction, a statistical technique for constructing prediction sets that are guaranteed to contain the semantically correct response with high probability. Additionally, TRAQ leverages Bayesian optimization to minimize the size of the constructed sets. In an extensive experimental evaluation, we demonstrate that TRAQ provides the desired correctness guarantee while reducing prediction set size by 16.2\% on average compared to an ablation. The implementation is available: \href{https://github.com/shuoli90/TRAQ.git}{https://github.com/shuoli90/TRAQ}. 
\end{abstract}

\section{Introduction}

Large Language Models (LLMs) have achieved State-Of-The-Art (SOTA) results on many question answering (QA) tasks~\cite{openai2023gpt4, touvron2023llama, touvron2023llama2}. However, in open-domain QA tasks where candidate answers are not provided, LLMs have also been shown to confidently generate incorrect responses, called \textit{hallucinations}~\cite{ouyang2022training, kuhn2023semantic}. Hallucinations have already led to real-world consequences when end users rely on the correctness of the generated text. As a consequence, there is an urgent need for techniques to reduce hallucinations.

\begin{figure}[t]
\centering
\includegraphics[width=0.5\textwidth]{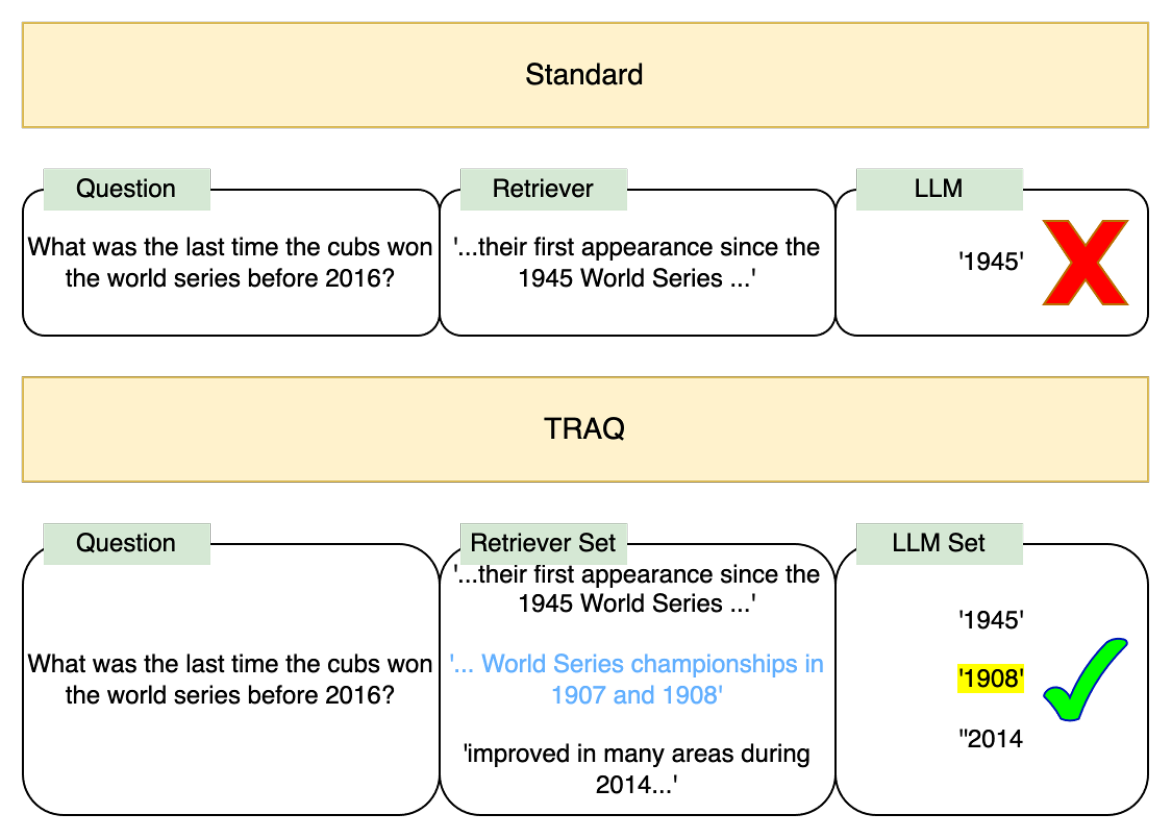}
\caption{Comparison of the standard RAG pipeline with TRAQ on a practical illustration reveals a significant difference. With the standard retrieval augmented generation (RAG) approach, there is a possibility that the retrieved passage may lack relevance in addressing the given question. On the contrary, TRAQ leverages conformal prediction to ensure that the retrieved set includes the relevant passage with a high probability and that the LLM set contains a semantically correct answer with a high probability. Through the aggregation of these prediction sets, TRAQ provides a guarantee that a semantically correct answer is contained in its set of answers with a high probability.
}
\label{fig:demo}
\end{figure}

We propose a novel framework, \emph{Trustworthy Retrieval Augmented Question Answering (TRAQ)}, summarized in Figure~\ref{fig:demo}, that combines Retrieval Augmented Generation (RAG)~\cite{guu2020retrieval,lewis2021retrievalaugmented} with conformal prediction~\cite{cp,shafer2007tutorial, park2020pac,angelopoulos2022gentle} to provide theoretical guarantees on question answering performance.

RAG reduces hallucinations by retrieving passages from a knowledge base such as Wikipedia and then using an LLM to answer the question. If the retrieved passages are relevant to the question, the LLM can use this information to generate correct answers. However, RAG can fail for two reasons: either the retrieved passage is not relevant to the question, or the LLM generates the incorrect answer despite being given a relevant passage.


To avoid these issues, TRAQ uses conformal prediction, an uncertainty quantification technique that modifies the underlying model to predict sets of outputs rather than a single output. These \emph{prediction sets} are guaranteed to contain the true output at a user-specified rate, e.g., at least 90\% of the time. In particular, TRAQ applies conformal prediction separately to the retrieval model (to obtain sets of retrieved passages guaranteed to contain the relevant passage with high probability) and the generator (to obtain sets of answers that contain the true answer with high probability, assuming the relevant passage is given). Then, TRAQ aggregates the two sets for the RAG task, as demonstrated in Figure~\ref{fig:aggregation}. By a union bound, retriever sets contain relevant passages, and generator prediction sets contain true answers with high probability, establishing that the aggregated set by TRAQ contains the ground truth answer with high probability.

A major challenge to this basic pipeline is that there may be many different ways of expressing the correct answer in natural language. For example, the responses \textit{deep learning is a subset of machine learning} and \textit{machine learning is a superset of deep learning} are different ways of expressing the same meaning~\cite{kuhn2023semantic, semanticclustering}. This diversity of possible responses also makes prediction probabilities less reliable since if an answer can be expressed in many different but equivalent ways, then the probabilities may be divided across these different responses, making them all smaller even if the model is confident it knows the correct answer.

TRAQ addresses this challenge by modifying the notion of ground-truth coverage in conformal prediction to focus on semantic notions of uncertainty. In particular, TRAQ aggregates semantically equivalent answers across a large number of samples from the LLM and uses the number of clusters of non-equivalent answers as a measure of uncertainty. This measure is used as a nonconformity measure to construct prediction sets. Finally, the prediction sets are over clusters of equivalent answers rather than individual answers. This strategy also enables TRAQ to work on black-box APIs such as \textit{GPT-3.5-Turbo}, where the predicted probabilities for individual tokens are not available.

A second challenge is that the prediction sets can become very large since we are aggregating uncertainty across multiple components. This complexity introduces hyperparameters into TRAQ; while TRAQ guarantees correctness regardless of the choice of these hyperparameters, they can affect the performance of TRAQ in terms of the average prediction set size. To address this challenge, TRAQ uses Bayesian optimization to minimize the average size of the prediction sets it generates.






We evaluate TRAQ in conjunction with several generative LLMs, including both GPT-3.5-Turbo-0613~\cite{ouyang2022training} and Llama-2-7B~\cite{touvron2023llama2}; and on four datasets, including a biomedical question answering dataset.
Our experiments demonstrate that TRAQ empirically satisfies the coverage guarantee (i.e. the prediction sets outputs contain semantically correct answers with the desired probability), while reducing the average prediction set size compared to an ablation by 16.2\%. Thus, TRAQ is an effective strategy for avoiding hallucinations in applications of LLMs to open domain question answering.

\paragraph{Contributions.}

We offer the first conformal prediction guarantees for retrieval augmented generation (RAG) targeted question answering. Our framework, TRAQ, introduces a novel nonconformity measure that estimates the uncertainty for each semantically distinct meaning and obtains a coverage guarantee at the semantic level. 
Furthermore, TRAQ leverages Bayesian optimization to minimize the average size of the generated prediction sets. Finally, our experiments demonstrate that TRAQ is effective at avoiding hallucinations in open-domain question answering.

\section{Background}

\paragraph{Retrieval for Open-Domain QA.}
A two-stage approach is often used for open-domain question answering (QA): first, a \textit{retriever} is used to obtain informative passages; and second, a \textit{generator} produces answers based on the retrieved passages.
A popular choice for the retriever is the Dense Passage Retriever (DPR)~\cite{karpukhin2020dense}, which measures similarity by taking the inner product of the BERT~\cite{devlin2019bert, reimers2019sentencebert} embeddings of the question and passage ~\cite{devlin2019bert, reimers2019sentencebert}. Other works~\cite{Lin2022ADR, Salemi2023ASD, Lin2022AggretrieverAS, Zhang2021AdversarialRF} have improved the performance of DPR and extended it to more diverse settings. Retrieval Augmented Generation (RAG)~\cite{lewis2021retrievalaugmented} proposes to jointly fine-tune the retriever and the generator for QA tasks.

\paragraph{Conformal Prediction.}
Conformal prediction~\cite{cp, papadopoulos2008inductive} is a general distribution-free approach to quantifying uncertainty for machine learning (ML) models. Let $\mathcal{X}$ be the input space, and $\mathcal{Y}$ be the output space. Conformal prediction first assumes that a \emph{nonconformity measure} (e.g., negative probabilities predicted by an ML model) $s: \mathcal{X}\times \mathcal{Y} \rightarrow \mathbb{R}$ is given. Lower values of $s(x, y)$ indicate better agreement between $x$ and $y$. Given a held-out calibration set $B = \{(x_i, y_i)\}_{i=1}^N$ sampled i.i.d. from the data distribution $\mathcal{D}$, as well as a user-specified error level $\alpha$, conformal prediction constructs a prediction set for a testing data point $X_{\text{test}}$ by 
\begin{equation} 
C(X_{\text{test}}) = \{y \in \mathcal{Y} \mid s(X_{\text{test}}, y) \leq \tau\}, 
\label{eq:cp}
\end{equation}
where $\tau$ is the $\frac{\lceil(1-\alpha)(N+1)\rceil}{N}$-th smallest score in $\{s(x_i, y_i)\}_{i=1}^N$. Conformal prediction guarantees that the true labels are contained in the constructed prediction sets with probability at least $1-\alpha$:
\begin{theorem}{Conformal Prediction Guarantee~\cite{angelopoulos2022gentle, shafer2007tutorial, cp}.} Suppose that $\{(x_i, y_i)\}_{i=1}^N$ and $(X_{\text{test}},Y_{\text{test}})$ are i.i.d. from $\mathcal{D}$, and $C(X_{\text{test}})$ is constructed by \eqref{eq:cp}; then, we have the following.
\begin{equation} 
\Pr_{(X_{\text{test}},Y_{\text{test}}) \sim \mathcal{D}}(Y_{\text{test}} \in C(X_{\text{test}})) \geq 1-\alpha. 
\end{equation} 
\label{th:cp} \end{theorem}
We call this guarantee a \emph{coverage guarantee}. An extension of conformal prediction is \textit{Probably Approximately Correct prediction sets}~\cite{park2019pac} (PAC prediction set) or \textit{training-conditional conformal prediction}~\cite{vovk2012conditional}. Compared with vanilla conformal prediction, where the coverage guarantee holds on average, PAC prediction sets guarantee that coverage is satisfied with high confidence given the current calibration set:
\begin{theorem}{PAC Guarantee~\citep{park2019pac, vovk2012conditional}.} Suppose $\{(x_i, y_i)\}_{i=1}^N$ and $(X_{\text{test}},Y_{\text{test}})$ are sampled i.i.d. from $\mathcal{D}$, given user-specified error and confidence levels $\alpha$ and $\delta$, and $C(X_{\text{test}})$ is constructed via \eqref{eqn:algorithm} in the Appendix; then, we have
\begin{equation*} 
\Pr_{B \sim \mathcal{D}^n} [\Pr_{(X,Y) \sim \mathcal{D}}(Y_{\text{test}} \in C(X_{\text{test}})) \geq 1-\alpha] \geq 1-\delta. \end{equation*}
\label{eq:pac}
\end{theorem} 
\noindent
Further details on conformal prediction and PAC prediction sets are in Appendices~\ref{sec:cp} \&~\ref{sec:pac}, respectively; a brief comparison between the two is given in Appendix~\ref{sec:cp_pac}. Both vanilla conformal prediction and PAC prediction sets have been applied to deep learning~\cite{park2019pac,https://doi.org/10.48550/arxiv.2009.14193,https://doi.org/10.48550/arxiv.2101.02703}.

\paragraph{Uncertainty Quantification for LLMs.}

Uncertainty quantification for Large Language Models (LLMs) has been gaining attention due to LLM hallucinations. A recent study~\cite{kuhn2023semantic} combined confidence calibration with Natural Language Inference model to measure the certainty of LLMs in responding to an input question. However, this work does not guarantee the accuracy of the responses. Other studies have applied conformal prediction to LLM predictions, focusing mainly on the multiple choice question answering problem and using vanilla conformal prediction to ensure correctness~\cite{kumar2023conformal, ren2023robots}. However, these methods necessitate a finite set of labels, such as \{\textit{True}, \textit{False}\} or \{\textit{A}, \textit{B}, \textit{C}\}, and cannot be used for open-domain question answering. A related work concurrent with ours is \citet{quach2023conformal}, which applies conformal prediction to open-domain QA.
However, they only consider the generator, whereas our approach provides conformal guarantees for RAG. Furthermore, their approach requires the generation probability from the LLM, which is not available in many blackbox APIs.

\section{The TRAQ Framework}

\begin{figure*}
\centering
\begin{subfigure}{0.43\textwidth}
\includegraphics[width=1.0\textwidth]{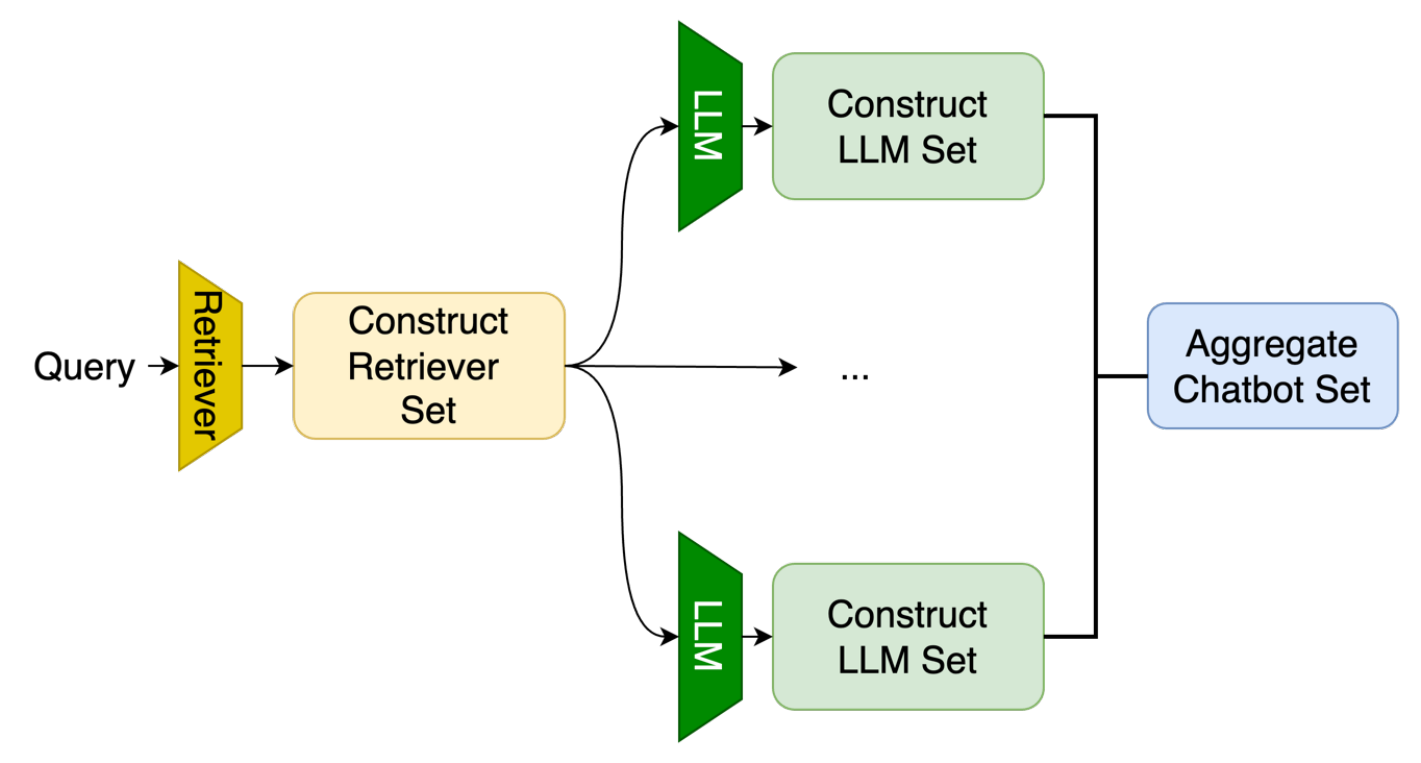}
\caption{Prediction Set Construction}
\label{fig:aggregation}
\end{subfigure}
~
\begin{subfigure}{0.55\textwidth}
\includegraphics[width=1.0\textwidth]{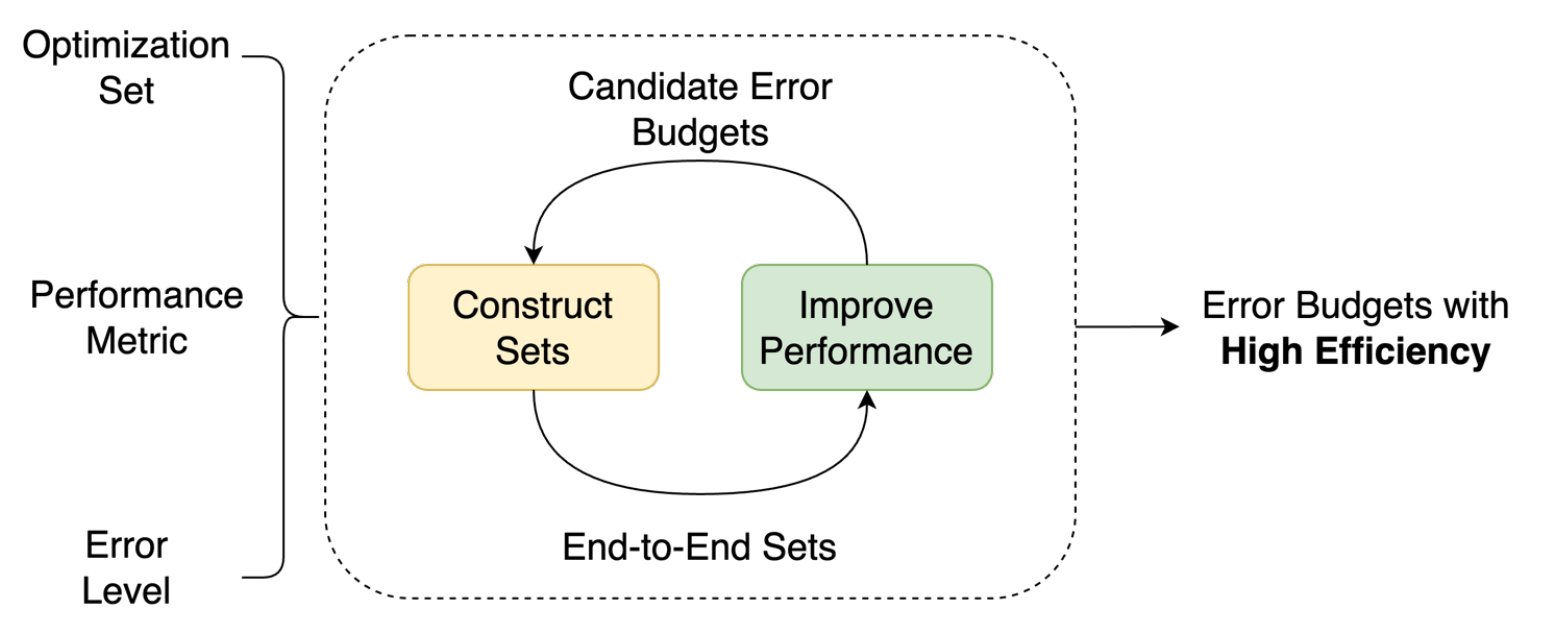}
\caption{Performance Improvement}
\label{fig:optimization}
\end{subfigure}

\caption{Given a question, TRAQ first constructs the retriever prediction; then, for every \textit{(question, contained passage)} pair, TRAQ constructs a LLM prediction on the LLM generated responses. Finally, the LLM prediction sets are aggregated as the final output. In Figure~\ref{fig:optimization}, TRAQ takes candidate error budgets from Bayesian optimization; it then constructs aggregated prediction sets on the optimization set. Next, the average semantic counts in constructed sets are computed to update the Gaussian process model in Bayesian optimization. }
\label{fig:diagrams}
\end{figure*}

TRAQ is composed of two steps. The first is the \textit{Prediction Set Construction} step, where a question $q$ is used to create a \textit{retrieval set} $C_{\text{Ret}}(q)$ for the retriever and a \textit{ LLM set} $C_{\text{LLM}}(q, p)$ for each pair $(\text{question}~q, \text{passage}~p)$. These sets are aggregated into an \textit{Aggregation Set} $C_{\text{Agg}}(q)$. The second step is the \textit{Performance Improvement} step, where promising error budgets $\alpha_{\text{Ret}}$ and $\alpha_{\text{LLM}}$ are sampled from a Bayesian model. Using these budgets, the prediction sets are constructed on the optimization set and evaluated for their performance. This process is repeated $T$ times, and the final output is the error budgets $\alpha_{\text{Ret}}$ and $\alpha_{\text{LLM}}$ with the highest performance. The chosen hyperparameters are used to construct prediction sets as in the first step using a separate held-out calibration set. The TRAQ framework is summarized in Figure~\ref{fig:diagrams}.

\subsection{Assumptions}

To construct provable prediction sets, we first make three necessary assumptions:
\begin{assumptionp}{I.I.D}
For both the retrieval and LLM tasks, the examples are drawn independently and identically from the data distribution $\mathcal{D}$.
\label{as:iid}
\end{assumptionp}

\begin{assumptionp}{\textit{Retriever Correctness}}
Given a question $q$, the underlying retriever is able to retrieve the most relevant passage $p^*$ within the top-$K$ retrieved passages.
\label{as:ret}
\end{assumptionp}

\begin{assumptionp} {\textit{LLM Correctness}}
Given a question $q$ and its most relevant passage $p^*$, the LLM is able to generate a semantically correct response within the top-$M$ samples.
\label{as:chat}
\end{assumptionp}

Assumption~\ref{as:iid} is a standard assumption from the conformal prediction literature and is needed to apply conformal prediction algorithms (it can be slightly relaxed to exchangeable distributions, but we make the i.i.d. assumption for simplicity).

Assumptions~\ref{as:ret} and~\ref{as:chat} are needed to ensure that the most relevant passages and semantically correct answers can be contained in the prediction sets if the prediction sets are sufficiently large. In principle, we can use very large values of $K$ and $M$ to satisfy this assumption, though there are computational and cost limitations in practice. We discuss ways to remove these assumptions in ~\nameref{sec:limitation}.

\subsection{Prediction Set Construction}
\label{sec:cp_sets}

\paragraph{Retriever Set:} To construct the retriever sets $C_{\text{Ret}}$, we use the negative inner product between the question $q$ and the annotated most relevant passage $p^*$, denoted as $-R_{q, p^*}$, as the nonconformity measures (NCMs). Given $N$ such NCMs $\{s_1, \ldots, s_N\}$ in the calibration set and the error budget $\alpha_{\text{Ret}}$ for the retriever set, we construct the retriever set by
\begin{equation}
C_{\text{Ret}}(q) = \{p \mid -R_{q, p} \leq \tau_{\text{Ret}}\},
\label{eq:set}
\end{equation} where $$\tau_{\text{Ret}} = \operatorname{Quantile}\left(\{s_n\}_{n=1}^N; \frac{\lceil(N+1)(1-\alpha_{\text{Ret}})\rceil}{N}\right).$$
\noindent
Given this construction and Assumptions~\ref{as:iid} and Assumption~\ref{as:ret}, the retriever sets are guaranteed to contain the most relevant passage with probability at least $1-\alpha_{\text{Ret}}$:
\begin{lemma}
Suppose the questions $q$ and their corresponding most relevant passage $p^*$ are sampled from the distribution $\mathcal{D_{\text{passage}}}$. Given the error budget $\alpha_{\text{Ret}}$, the retriever sets satisfy
$$\Pr_{(q, p^*) \sim \mathcal{D_{\text{Passage}}}} (p^* \in C_{\text{Ret}}(q)) \geq 1-\alpha_{\text{Ret}}.$$
\label{co:ret}
\end{lemma} 
\noindent
This result follows straightforwardly from Theorem~\ref{th:cp} and Assumptions~\ref{as:iid} \&~\ref{as:ret}. We give a proof in Appendix~\ref{pf:ret}.

\paragraph{LLM Set:} We utilize Monte Carlo sampling to approximate confidences for different semantic meanings; then, we use the negative approximated confidences as the NCMs to construct LLM sets. Specifically, for each \textit{(question, passage)} pair, we ask the LLM to generate $M$ responses ($M=30$ in our experiments). Given two responses $r$ and $r^{\prime}$, we cluster them together if they have high similarity, which is measured by Rouge score~\cite{lin-2004-rouge} or Natural Language Inference (NLI) model ~\cite{kuhn2023semantic, he2021deberta}. We consider the two responses to be semantically similar if they have a Rouge score greater than 0.7 or are deemed to entail each other by the NLI model. After clustering, for each cluster $i$, let $N_i$ be the number of responses in the cluster; 
we approximate the confidence of a response $r$ by \textbf{$N_i / M$} if $r$ belongs to the $i$-th cluster.
Finally, given the error budget for LLM $\alpha_{\text{LLM}}$, we can utilize a similar process to that in \eqref{eq:set} to construct LLM sets. The constructed sets satisfy the following:
\begin{lemma}
Suppose the questions $q$, their corresponding most relevant passage $p^*$, and semantically correct responses $r^*$ are sampled from distribution $\mathcal{D_{\text{Response}}}$. Given error budget $\alpha_{\text{LLM}}$, if Assumptions~\ref{as:iid} \&~\ref{as:chat} hold, the LLM sets satisfy
$$\Pr_{(q, p^*, r^*) \sim \mathcal{D_{\text{Response}}}} (r^* \in C_{\text{LLM}}(q, p^*)) \geq 1-\alpha_{\text{LLM}}.$$
\label{co:chat}
\end{lemma}
\noindent
The proof of Lemma~\ref{co:chat} is similar to that of Lemma~\ref{co:ret}; we give it in Appendix~\ref{pf:chat}.

Note that since the uncertainty score can be arbitrary in conformal prediction, the lemma~\ref{co:chat} holds regardless of the chosen heuristic measures (e.g., Rouge score or BERT embedding). If the chosen heuristic underperforms, conformal prediction will simply construct large prediction sets to compensate. We validate this claim in Section~\ref{sec:metrics}.

\paragraph{Aggregated Set:} To obtain an overall correctness guarantee, we construct an aggregated set $C_{\text{Agg}}$ by constructing an LLM set $C_{\rm LLM}$ for each passage $q$ contained in the retriever set; and take the union of the $C_{\rm LLM}$'s, i.e.
\begin{equation}
    C_{\rm Agg}(q) = \cup_{p \in C_{\rm Ret}(q)} C_{\rm LLM}(q, p).
    \label{eq:agg}
\end{equation}
Then, the resulting Aggregated set $C_{\text{Agg}}$ satisfies the following:
\begin{theorem}
Suppose the questions $q$ and semantically correct responses $r^*$ are sampled from the distribution $\mathcal{D}$, and a user-specified error level $\alpha$ is given. By aggregating retriever sets with error budget $\alpha_{\text{Ret}}$ by \eqref{eq:agg} with LLM sets with error budget $\alpha_{\text{LLM}}$, with $\alpha=\alpha_{\text{Ret}} + \alpha_{\text{LLM}}$, the aggregated sets satisfy
$$\Pr_{(q, r^*) \sim \mathcal{D}} (r^* \in C_{\text{Agg}}(q)) \geq 1-\alpha.$$
\label{th:e2e}
\end{theorem}
\noindent
We give a proof in Appendix~\ref{pf:e2e}. After taking the union, we remove duplicated responses and re-cluster semantic meanings. Given that this post-processing phase solely eliminates duplicate responses, it will not remove correct semantic meanings, and Theorem~\ref{th:e2e} remains valid.

Note that this aggregation process is actually a global hypothesis testing method called the Bonferroni correction. Lemmas~\ref{co:ret} \&~\ref{co:chat} and Theorem~\ref{th:e2e} can be straightforwardly extended to the probably approximately correct (PAC) guarantee by constructing PAC prediction sets; see Appendix~\ref{sec:pac_set} for details. 

\subsection{Performance Improvement}

By Theorem~\ref{th:e2e}, we can guarantee that semantically correct responses are included in the aggregated set with a probability of at least $1-\alpha$, assuming $\alpha=\alpha_{\text{Ret}}+\alpha_{\text{LLM}}$. This theorem is valid for any combination of the two error budgets. However, the predictive performance of the aggregation sets is influenced by the specific choice of the error budgets. This issue has been discussed in the Bonferroni correction and the global testing literature~\cite{NEUWALD1994698, HMP, Poole029637}.

Therefore, we optimize the error budgets using Bayesian optimization, a sampling-based global optimization technique suitable for non-convex, non-closed-form problems; see Appendix~\ref{sec:bayesian} for details. In TRAQ, Bayesian optimization first models the underlying performance landscape using a Gaussian process; then, it samples error budgets (i.e., $\alpha_{\text{Ret}}$ and $\alpha_{\text{LLM}}$) based on the Gaussian process, and identifies $\tau_{\rm Ret}$ and $\tau_{\rm LLM}$ on a held-out optimization set $B_{\rm Opt}$. After assessing the performance of the sampled error budgets on $B_{\rm Opt}$, the Gaussian process is modified to more accurately reflect the performance landscape. This process is repeated $T$ times. The pseudocode for this procedure is shown in Algorithm~\ref{alg:bayopt_appendix}.

\begin{algorithm}[t]
\caption{Prediction Set Optimization}
\label{alg:bayopt_appendix}
\begin{algorithmic}[1]
\Statex {\bfseries Input:}
Optimization set $B_{\text{Opt}}$, performance metric $f$, error level $\alpha$
\State Initialize Gaussian process $G$
\For{$t\in\{1,...,T\}$}
\State Sample $\alpha_{\text{Ret}}$ and $\alpha_{\text{LLM}}$ basing on $G$
\State Normalize $\alpha_{\rm Ret}$ and $\alpha_{\rm LLM}$ so that 
$\alpha_{\rm Ret}, \alpha_{\rm LLM} \in (0,1)$, and $\alpha_{\rm Ret}+\alpha_{\rm LLM}=\alpha$
\State Compute $\tau_{\text{Ret}}$ and $\tau_{\text{LLM}}$ on $B_{\text{Opt}}$
\State Construct $C_{\text{Agg}}$ on $B_{\text{Opt}}$
\State Evaluate performance of the $C_{\rm Agg}$ using $f$
\State Update $G$ using the evaluation results
\EndFor
\State {\bfseries return:} the best error budgets $\alpha_{\text{Ret}}$ and $\alpha_{\text{LLM}}$
\end{algorithmic}
\end{algorithm}

\section{Experiments}

\paragraph{Experiment Setup.}
We evaluate TRAQ on four datasets, including three standard QA datasets (Natural Question~\cite{kwiatkowski-etal-2019-natural}, TriviaQA~\cite{joshi-etal-2017-triviaqa}, SQuAD-1~\cite{rajpurkar2016squad}), and one biomedical QA dataset (BioASQ~\cite{TSA+12}). On each dataset, we collect 1,000 samples that met the criteria of Assumptions~\ref{as:ret} \&~\ref{as:chat}. We divide each dataset into calibration, optimization, and testing sets, with 300, 300, and 400 data points, respectively.

We employ two fine-tuned DPR models, one~\cite{karpukhin-etal-2020-dense} trained on the Natural Question, TriviaQA, and SQuAD-1 datasets, and the other fine-tuned on BioASQ (see Appendix~\ref{sec:ft_bio} for training details). Furthermore, we use two generative large language models (LLMs): \textit{GPT-3.5-Turbo-0613} (\textit{GPT-3.5}), whose internal embedding and prediction probabilities are not accessible, and \textit{Llama-2-7B} (\textit{Llama-2}). We separately fine-tune Llama-2 on Natural Question, TriviaQA, and SQuAD-1, with hyperparameters given in Appendix~\ref{sec:llama_ft}.

For each question, we retrieve the top-20 passages; for each \textit{(question, passage)} pair, we sample 30 responses, with a temperature of 1.0.

We evaluate using coverage levels 50\%, 60\%, 70\%, 80\%, and 90\%. For the PAC guarantee, we use confidence level 90\%. We use five random seeds for each experiment. To investigate the influence of prompt design, we design two prompts, one zero-shot and one few-shot prompt; the few-shot prompt includes two demonstrations. The prompt templates are provided in Appendix~\ref{zero}. Unless otherwise specified, the zero-shot prompt is used for both GPT-3.5 and Llama-2. 

We evaluate the performance of our approach using two metrics. The first metric is \emph{coverage rate}, which is the rate at which the correct responses are contained in the constructed sets. We consider the responses to be \textit{correct} if their \textit{Rouge-1}~\cite{lin-2004-rouge} scores with the annotated answers are greater than 0.3. The coverage rate is expected to be no less than the desired level on average across different random seeds. The second metric is the average prediction set size. Specifically, we consider two size measures: (i) the average number of semantic clusters and (ii) the average number of unique answers. Lower values indicate better performance. 


We compare our approaches, \textit{TRAQ} and \textit{TRAQ-P} (the PAC version), to several baselines, including \textit{Vanilla}, \textit{Bonf}, and \textit{Bonf-P}. Vanilla is a baseline that does not construct prediction sets and only uses the top retrieved passage and generated answers. Bonf and Bonf-P are ablations that omit Bayesian optimization.
In all plots, we also show the \textit{Reference} line indicating the desired coverage level.

We report both quantitative and qualitative results. Our quantitative experiments aim to answer the following. \\
\textit{$\left(Q_1\right)$ Do the coverage guarantees hold for the retriever and the generator? \\}
\textit{$\left(Q_2\right)$ Does the overall coverage guarantee hold? \\}
\textit{$\left(Q_3\right)$ How do Bayesian optimization and the coverage level affect prediction set sizes? \\}
\textit{$\left(Q_4\right)$ Does TRAQ work for different semantic clustering methods and performance metrics? \\}
\textit{$\left(Q_5\right)$ How does prompt affect results?}

\paragraph{Q1: Do the coverage guarantees hold for the retriever and generator?}

To validate the coverage guarantees of the retriever and generator, we consider the coverage rates of retriever and LLM sets (named \textit{Ret} and \textit{LLM}), and with the PAC guarantee (named \textit{Ret-P} and \textit{LLM-P}). We report results on BioASQ using GPT-3.5 in Figure~\ref{fig:individual_cov}; Results for other datasets and different LLMs are reported in
Figure~\ref{fig:additional_individual}, and are qualitatively similar.
As shown in Figure~\ref{fig:individual_cov}, the empirical coverage levels of the retrieval and QA prediction sets are close to the desired coverage levels. Thus, the coverage guarantees hold for individual components, as desired.

We also report empirical coverage rates with 20 random seeds in Figure~\ref{fig:additional_individual_20}. Compared to results with 5 random seeds, empirical coverage with more random seeds become closer to the desired level. Furthermore, when using the PAC prediction sets, the empirical coverage levels were almost always above the expected coverage levels across all random seeds, as desired.

\begin{figure}[t]
\centering
\includegraphics[width=0.45\textwidth]{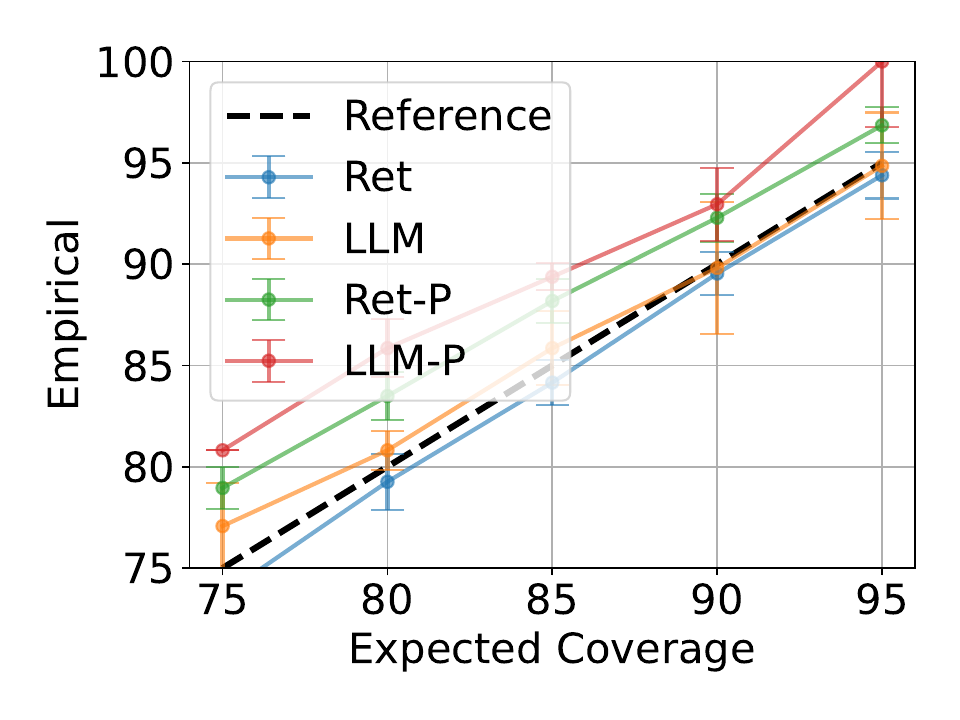}
\caption{Retriever and generator coverage rates on the BioASQ dataset.}
\label{fig:individual_cov}
\end{figure}

\paragraph{Q2: Does the end-to-end coverage guarantee hold?}

To verify the end-to-end guarantees from TRAQ, we report two rates. The first is the rate at which the correct responses are covered considering only the annotated most relevant passages:
\begin{equation*}
\Pr(p^{*} \in C_{\text{Ret}}(q)) \times \Pr(r^{*} \in C_{\text{LLM}}(q, p^{*})).
\end{equation*}
These results are shown in Figure~\ref{fig:e2e1_bio}. They show that the rates on average satisfy the desired coverage levels when using conformal prediction. In addition, the rates are mostly above the desired coverage levels when using PAC prediction sets. Second, we report the rate at which the correct responses are covered in the aggregated prediction set.
\begin{equation*}
\Pr(r^{*} \in C_{\text{Agg}}(q)).
\end{equation*}  The results are shown in Figure~\ref{fig:e2e2_bio}. Different from Figure~\ref{fig:e2e1_bio}, empirical levels of both conformal prediction and PAC prediction sets are above the expected coverage levels most of the time. This is because the generator might output the correct response even if it is not given a relevant passage.

\begin{figure}[t]
\centering
\includegraphics[width=0.45\textwidth]{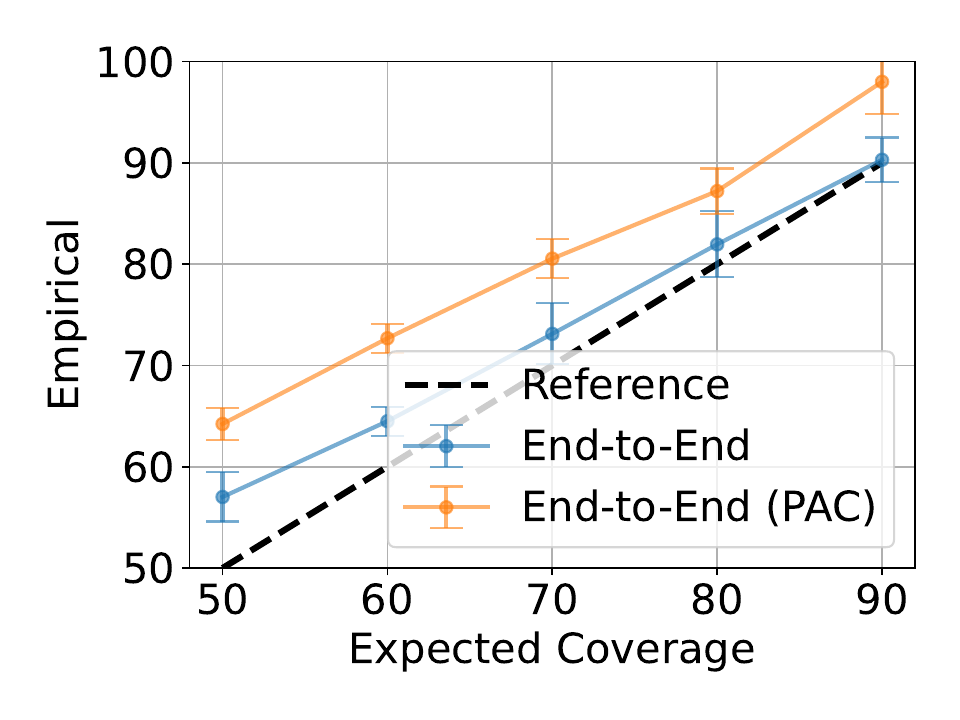}
\caption{End-to-end guarantee considering only the most relevant passage on BioASQ Dataset.}
\label{fig:e2e1_bio}
\end{figure}

\begin{figure}[t]
\centering
\includegraphics[width=0.45\textwidth]{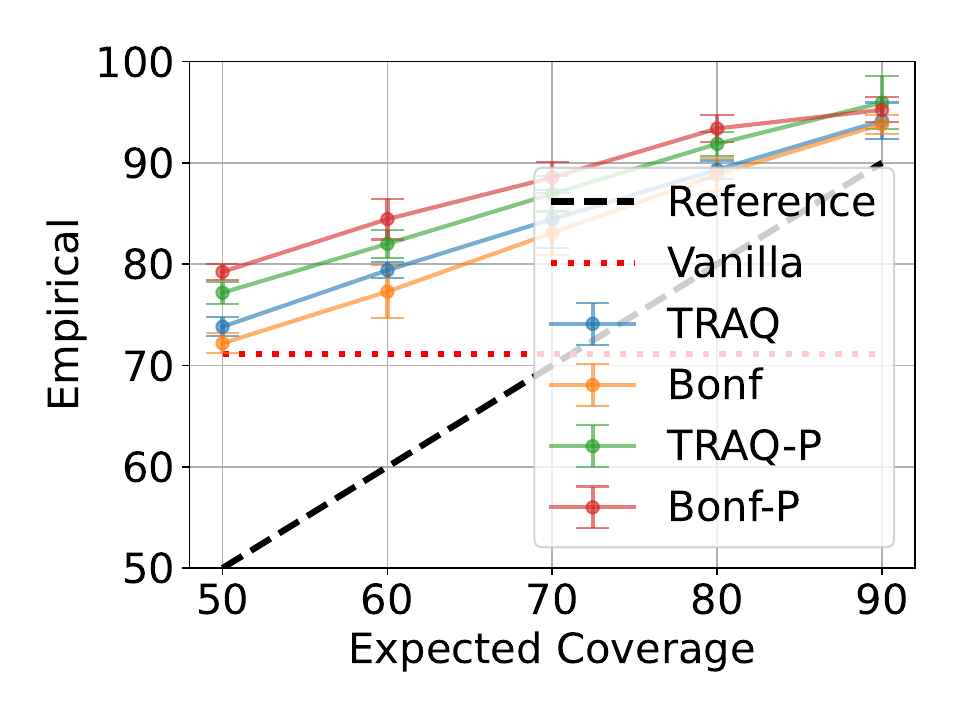}
\caption{End-to-end coverage guarantee considering all passages on the BioASQ dataset.}
\label{fig:e2e2_bio}
\end{figure}

\paragraph{Q3: How do Bayesian optimization and the coverage level affect prediction set sizes?}
\label{sec:efficiency}
To demonstrate the advantages of incorporating Bayesian optimization, we evaluate the average prediction set sizes (in terms of the number of semantic clusters) across different approaches. We first show results across different coverage levels and random seeds using different methods on BioASQ dataset Figure~\ref{fig:avg_size}. It shows that TRAQ and TRAQ-P are able to construct smaller prediction sets than their counterparts without Bayesian optimization (Bonf and Bonf-P). Furthermore, we report the average semantic counts on different datasets and coverage levels using GPT-3.5 in Table~\ref{tab:all_semantic_chatgpt} and using Llama-2 in Table~\ref{tab:all_semantic_llama}. As can be seen, Bayesian optimization is especially effective in reducing prediction set size when higher coverage rates are desired (80\% and 90\%). In these cases, both TRAQ and TRAQ-P are able to construct significantly smaller prediction sets, reducing their size by 16.2\% on average (18.1\% in Table~\ref{tab:all_semantic_chatgpt} and 14.2\% in Table~\ref{tab:all_semantic_llama}).
Importantly, even though the prediction sets are smaller, the desired overall coverage guarantees still hold. These tables also show that higher coverage levels tend to result in larger prediction set sizes; this trade-off is expected since stronger statistical guarantees require more conservative prediction sets.

\begin{figure}[t]
\centering
\includegraphics[width=0.45\textwidth]{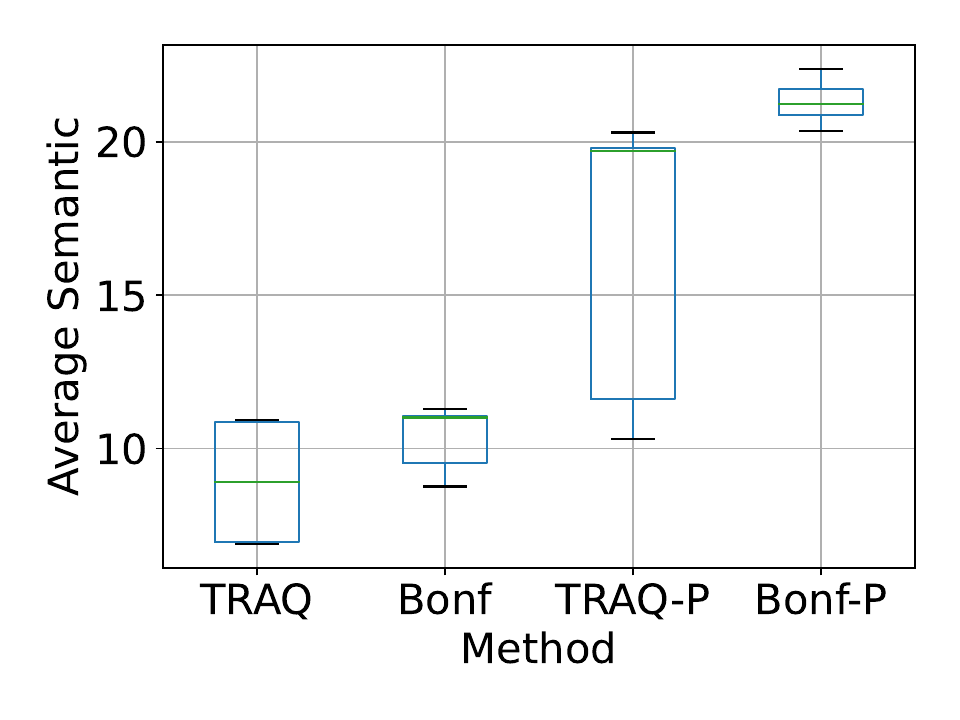}
\caption{Prediction set sizes according to the average number of semantic clusters.}
\label{fig:avg_size}
\end{figure}

\begin{table}
\centering
\resizebox{\columnwidth}{!}{%
\begin{tabular}{llllll}
\toprule
 Task & Cov(\%)  & TRAQ & Bonf & TRAQ-P & Bonf-P \\
\midrule
\multirow[t]{5}{*}{BIO} & 50 & $2.5_{0.1}$ & $\mathbf{2.4_{0.1}}$ & $2.9_{0.1}$ & $2.9_{0.2}$ \\
 & 60 & $\mathbf{2.9_{0.2}}$ & $\mathbf{2.9_{0.2}}$ & $3.4_{0.1}$ & $3.6_{0.2}$ \\
 & 70 & $\mathbf{3.5_{0.2}}$ & $3.6_{0.2}$ & $4.0_{0.3}$ & $4.6_{0.1}$ \\
 & 80 & $\mathbf{4.4_{0.2}}$ & $5.0_{0.2}$ & $5.8_{0.6}$ & $7.2_{0.5}$ \\
 & 90 & $\mathbf{8.9_{2.0}}$ & $10.3_{1.1}$ & $16.3_{4.9}$ & $21.3_{0.8}$ \\
\cline{1-6}
\multirow[t]{5}{*}{NQ} & 50 & $\mathbf{3.0_{0.3}}$ & $3.2_{0.2}$ & $3.6_{0.2}$ & $3.7_{0.1}$ \\
 & 60 & $\mathbf{3.7_{0.1}}$ & $\mathbf{3.7_{0.1}}$ & $4.5_{0.2}$ & $4.4_{0.1}$ \\
 & 70 & $\mathbf{4.6_{0.3}}$ & $\mathbf{4.6_{0.2}}$ & $5.7_{0.5}$ & $5.7_{0.2}$ \\
 & 80 & $\mathbf{6.1_{0.5}}$ & $6.4_{0.2}$ & $7.3_{0.6}$ & $9.3_{1.1}$ \\
 & 90 & $\mathbf{10.3_{2.7}}$ & $12.2_{1.5}$ & $16.7_{4.6}$ & $23.6_{0.6}$ \\
\cline{1-6}
\multirow[t]{5}{*}{Trivia} & 50 & $\mathbf{2.0_{0.2}}$ & $\mathbf{2.0_{0.1}}$ & $2.4_{0.4}$ & $2.4_{0.1}$ \\
 & 60 & $2.5_{0.3}$ & $\mathbf{2.4_{0.1}}$ & $2.9_{0.4}$ & $2.7_{0.2}$ \\
 & 70 & $3.0_{0.4}$ & $\mathbf{2.9_{0.2}}$ & $3.5_{0.3}$ & $3.4_{0.2}$ \\
 & 80 & $\mathbf{3.7_{0.3}}$ & $3.8_{0.3}$ & $4.6_{0.3}$ & $4.6_{0.3}$ \\
 & 90 & $5.9_{0.6}$ & $\mathbf{5.8_{0.4}}$ & $7.2_{0.9}$ & $7.8_{0.3}$ \\
\cline{1-6}
\multirow[t]{5}{*}{SQuAD1} & 50 & $3.6_{0.1}$ & $\mathbf{3.5_{0.0}}$ & $4.1_{0.2}$ & $4.0_{0.1}$ \\
 & 60 & $\mathbf{4.1_{0.2}}$ & $\mathbf{4.1_{0.1}}$ & $4.6_{0.1}$ & $5.0_{0.1}$ \\
 & 70 & $\mathbf{4.8_{0.2}}$ & $5.2_{0.2}$ & $5.5_{0.3}$ & $7.4_{0.2}$ \\
 & 80 & $\mathbf{6.2_{0.6}}$ & $8.2_{0.3}$ & $8.9_{1.3}$ & $11.0_{0.2}$ \\
 & 90 & $\mathbf{12.6_{2.1}}$ & $14.1_{0.4}$ & $21.3_{5.6}$ & $25.9_{0.5}$ \\
\cline{1-6}
\bottomrule
\end{tabular}
}
\caption{Average semantic counts using GPT-3.5.}
\label{tab:all_semantic_chatgpt}
\end{table}

\paragraph{Q4: Does TRAQ work with different semantic clustering methods?} \label{sec:metrics}
We evaluate whether TRAQ remains effective with different semantic clustering methods and performance metrics. We use the semantic clustering method proposed by~\citet{kuhn2023semantic}, which is based on BERT~\cite{devlin2019bert, reimers2019sentencebert}, and specified the performance metric as the average number of unique answers in the aggregated prediction sets. We evaluate this setup on the SQuAD-1 dataset using GPT-3.5. The results, shown in Figures~\ref{fig:cov_semantic} \&~\ref{fig:eff_semantic}, demonstrate that TRAQ remains successful. Specifically, Figure~\ref{fig:cov_semantic} shows that the overall coverage guarantee holds, and Figure~\ref{fig:eff_semantic} demonstrates that TRAQ and TRAQ-P reduce prediction set sizes compared to their ablations Bonf and Bonf-P, respectively.

\begin{figure}[t]
\centering
\includegraphics[width=0.45\textwidth]{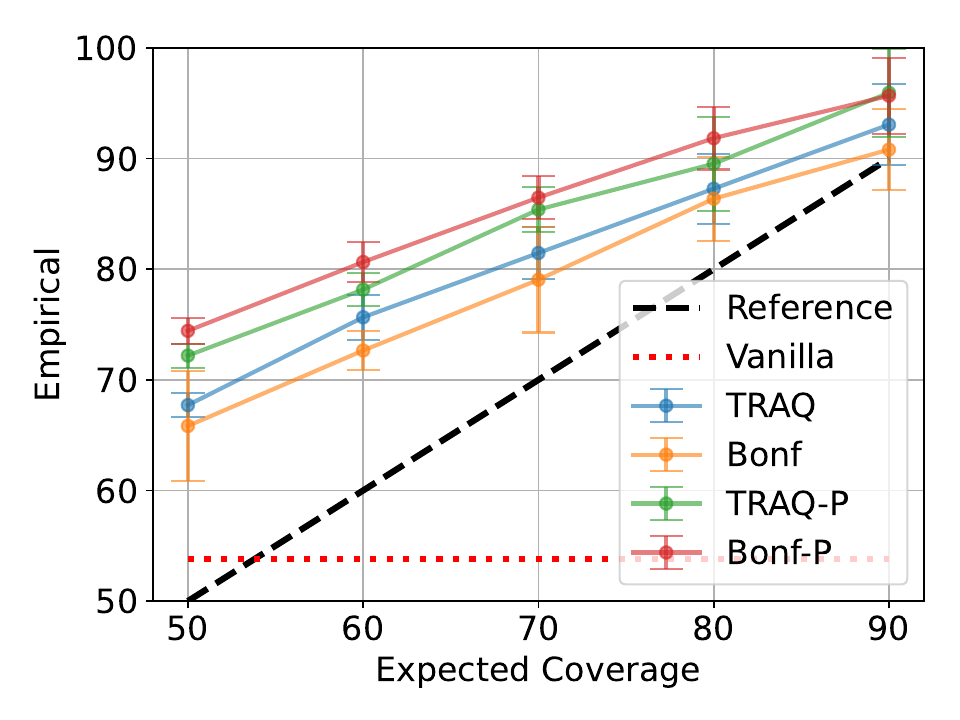}
\caption{Coverage rate using BERT embeddings on SQuAD-1 dataset.}
\label{fig:cov_semantic}
\end{figure}

\begin{figure}[t]
\centering
\includegraphics[width=0.45\textwidth]{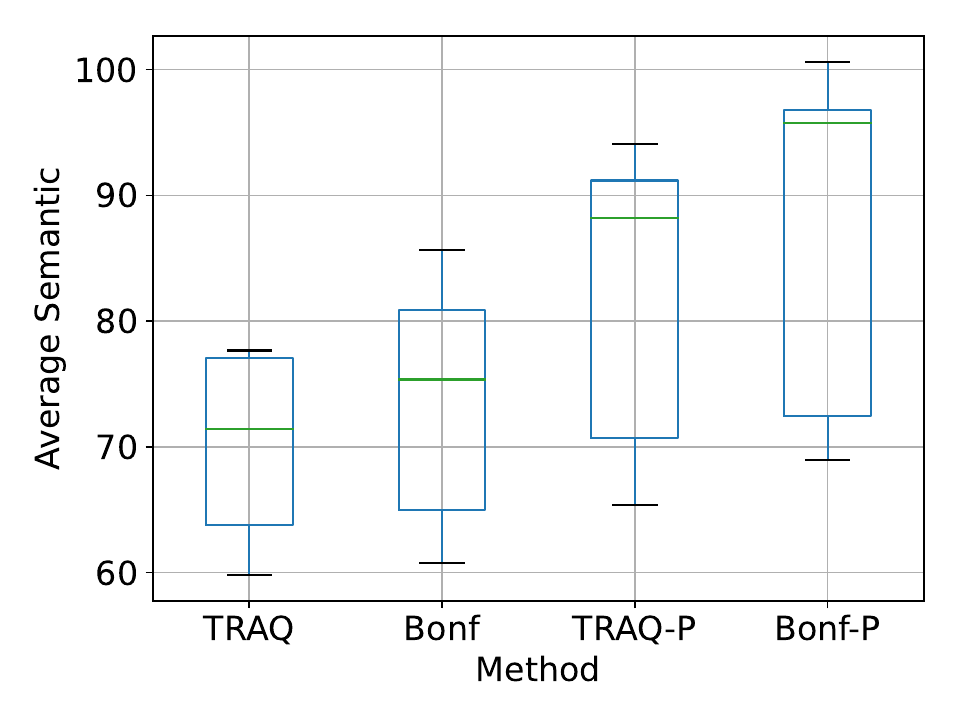}
\caption{Prediction set size according to average number of unique answers.}
\label{fig:eff_semantic}
\end{figure}

\paragraph{Q5: How does prompt engineering affect results?}
We investigate how prompt engineering affects TRAQ performance using a few-shot prompt with two demonstrations. The prompt template is provided in Appendix~\ref{fewshot}. We evaluate TRAQ on Natural Question using GPT-3.5. The end-to-end coverage rates and prediction set sizes using different methods are shown in Figure~\ref{fig:additional_semantic}. TRAQ with a few shot prompt achieves the desired coverage rate on average and reduces prediction set size compared to its ablation. In Figure~\ref{fig:zs_fs}, we also compare the zero-shot and few-shot prompts in terms of performance. Interestingly, zero-shot prompting mostly yields better efficiencies. This could be because zero-shot prompting generated more diverse answers and had lower confidence in wrong answers. An example of the comparison between responses using different prompts is given in the Appendix~\ref{sec: prompts}.

\begin{figure}[t]
\centering
\includegraphics[width=0.45\textwidth]{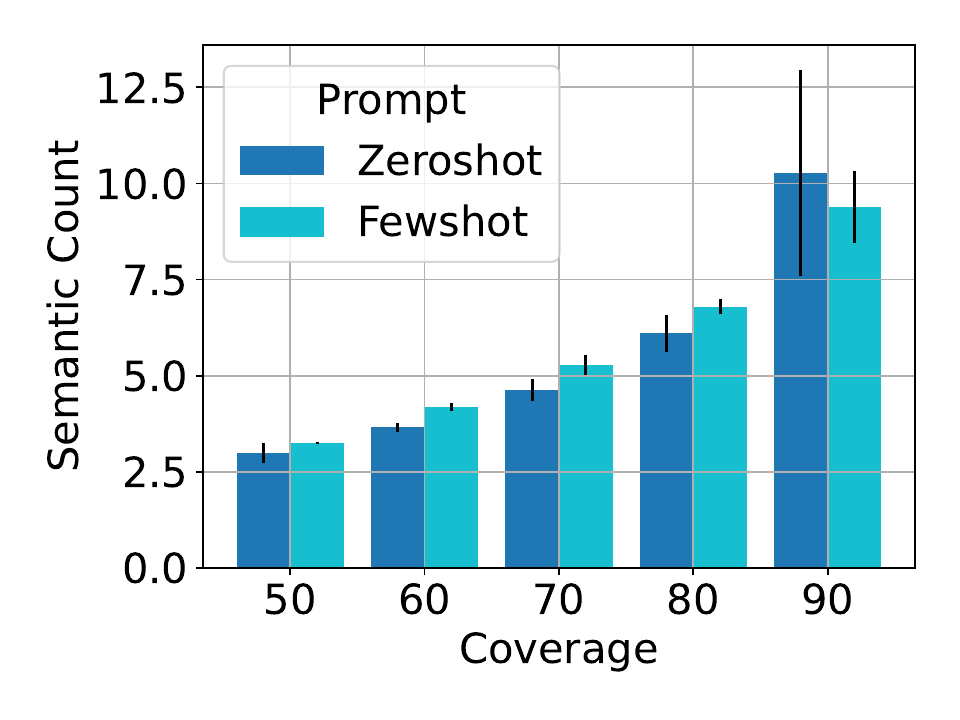}
\caption{Comparison between zero-shot and few-shot prompts on prediction set size.}
\label{fig:zs_fs}
\end{figure}

\paragraph{Qualitative Analysis.}
By constructing prediction sets, TRAQ guarantees that it includes correct responses with high probability. For example, we consider the following question: \textit{Who played in the movie a star is born with Judy Garland?}, where \textit{James Mason} is a correct answer. The responses of different methods are shown below. While standard RAG fails to return the correct answer, TRAQ and Bonf output sets containing the correct answers; and TRAQ obtains a smaller set.
%
%
%
%
%
%
\begin{lstlisting}

Question: who played in the movie a star is born with judy garland

True Answers: {'James Mason', 'Charles Bickford', 'Jack Carson'}

Standard: {'Gary Busey', 'Judy Garland', 'Barbra Streisand'}

TRAQ: {'Judy Garland',  'James Mason', 'Lady Gaga', 'Sid Luft', 'Danny Kaye'}

Bonf {'Gary Busey', 'Judy Garland',  'James Mason', 'Lady Gaga', 'Bradley Cooper', 'Sidney Luft', 'Danny Kaye'}
\end{lstlisting}
We show additional examples in Appendix~\ref{sec:qual}.










\section{Conclusion}
\label{sec:conclusion}
We propose an algorithm, called \emph{Trustworthy Retrieval Augmented Question Answering (TRAQ)}, which applies conformal prediction to construct prediction sets for Retrieval Augmented Generation (RAG). TRAQ first constructs prediction sets for the retriever and generator and then aggregates these sets. TRAQ guarantees that for each question, a semantically correct answer is included in the prediction set it outputs with high probability. To the best of our knowledge, this guarantee is the first conformal guarantee for retrieval augmented generation. Additionally, to minimize prediction set size, TRAQ leverages Bayesian optimization to identify optimal hyperparameters. In our comprehensive experiments, we demonstrate that TRAQ provides an overall semantic level coverage guarantee across different tasks, and that Bayesian optimization effectively reduces prediction set size.

\section{Broader Impacts}
\label{sec:impact}

The need for trustworthy AI algorithms has recently become paramount due to the risks of spreading misleading information~\cite{Presiden9:online, Commissi41:online}. We propose TRAQ, a framework that aims to address the hallucination problem by using conformal prediction to provide probabilistic guarantees for retrieval augmented generation (RAG). In addition, TRAQ leverages novel techniques to improve performance that may be useful more broadly in conformal prediction.

\section{Limitations}
\label{sec:limitation}
TRAQ makes three assumptions: that the data is independent and identically distributed (\ref{as:iid}), that the retriever has good performance (\ref{as:ret}), and that the language model can generate a response to the input question (\ref{as:chat}). Our experiments have verified \ref{as:iid}, but \ref{as:ret} and \ref{as:chat} may not be valid if the underlying retriever and language model do not perform well. To relax \ref{as:ret}, we can select more passages than the top-\textbf{20} used in our experiments. To remove \ref{as:chat}, we propose providing a guarantee of including \textit{I do not know} in the aggregation set if the language model cannot answer the input question. We describe how TRAQ can be modified to provide such guarantees in Appendix~\ref{sec: idk}.

TRAQ is a post-hoc method, so its prediction sets may be larger than necessary if the underlying models, such as the retriever and large language model, do not work properly. Additionally, if the semantic clustering techniques (Rouge score based or BERT-based) are invalid, then some semantically unrelated answers may be aggregated.

Finally, TRAQ can reduce inference speed due to the need for multiple retrievals, each of which needs to be embedded separately by the LLM. In our current setup, the computational complexity of the retrieval phase increases linearly with the number of retrievals (typically around 15). Avoiding this overhead is a key direction for future research.

\section*{Ackowledgement}
This work was generously supported by NSF Award CCF-1917852, ARO Award W911NF20-1-0080, and 
Institute of Information \& communications Technology Planning \& Evaluation (IITP) grant funded by the Korea government (MSIT) (No.2019-0-01906, Artificial Intelligence Graduate School Program (POSTECH)).

\bibliography{custom}

\newpage
\appendix
\onecolumn

\section{Conformal Prediction and PAC Guarantees}

\subsection{Conformal Prediction and Hypothesis Testing}
\label{sec:cp}
Conformal prediction is a distribution-free uncertainty quantification technique that constructs provable prediction sets for black-box models. Specifically, let $\mathcal{X}$ and $\mathcal{Y}$ be the input and label spaces, respectively, and $(x,y)$ be an input-label pair. Conformal prediction assumes given a calibration set $B=\{x_i,y_i\}_{i=1}^{N}$ with $N$ input-label pairs, along with a \textit{nonconformity measure} $s: \mathcal{X} \times \mathcal{Y} \rightarrow \mathbb{R}$ that measures how different a pair $(x,y)$ is from the examples sampled from the distribution $\mathcal{D}$. Given a new input $x_{\text{test}}$, conformal prediction constructs a prediction set $C(x_{\text{test}})\subseteq\mathcal{Y}$ using Algorithm~\ref{alg:cp}. Intuitively, for each label $y \in \mathcal{Y}$, this algorithm checks whether $(x_{\text{test}},y)$ is similar to the examples in $B$ according to the nonconformity measure $s(x_{\text{test}}, y)$. If $s(x, y)$ is low enough, then $y$ is included in the prediction set $C(x_{N+1})$; otherwise, $y$ is excluded from $C(x_{N+1})$. 

\begin{algorithm}[tbh]
\caption{The Conformal Algorithm}
\label{alg:cp}
\begin{algorithmic}
\State {\bfseries Input:} Nonconformity measure $s$, significance level $\alpha$, calibration st $B=\{x_n, y_n\}_{n=1}^N$, a new input $x_{\text{test}}$, label space $\mathcal{Y}$
\State Compute the threshold $\tau$ as the $\frac{\lceil(1-\alpha)(N+1)\rceil}{N}$-th smallest score in $\{s(x_i, y_i)\}_{i=1}^N$.
\State Construct prediction set for $x_{\text{test}}$ by 
\begin{equation*}
C(x_{\text{test}}) = \{y \mid s(x_{\text{test}}, y), y \in \mathcal{Y}\}
\end{equation*}
\State {\bfseries Return:} $C(x_{\text{test}})$.
\end{algorithmic}
\end{algorithm}

\subsection{PAC Prediction Set}
\label{sec:pac}
PAC prediction sets~\cite{vovk2012conditional,park2021pac} are a variant of conformal prediction approach that satisfies stronger PAC-style guarantees. Let $\mathcal{D}$ be the distribution of samples, and $B=\{x_n, y_n\}_{n=1}^N$ be a held-out calibration set of i.i.d. data points from $\mathcal{D}$ of size $N$. We denote the joint distribution on N samples by $\mathcal{D}^N$. The goal is to find a set of a small size satisfying the PAC property, that is, given $\alpha,\delta\in(0,1)$, 
\begin{equation*}
\Pr_{Z \sim \mathcal{D}^n}[ L_\mathcal{D}(C) \leq \alpha ] \ge 1 - \delta,
\end{equation*} 
where the $\Pr_{Z \sim \mathcal{D}^n}$ refers to the chances of calibration succeeding. In this case, we say $C$ is \emph{$(\alpha, \delta)$-probably approximately correct (PAC)}. To construct $(\alpha, \delta)$-PAC sets, the PAC prediction set considers the following one-dimensional parameterization of the prediction sets:
\begin{equation*}
C_\tau(x) = \{ y \in \mathcal{Y} \mid g(x, y) \geq \tau \},
\end{equation*}
where $\tau \geq 0$ and $g:\mathcal{X} \times \mathcal{Y} \rightarrow \mathbb{R}_{\geq 0}$ is any given scoring function (e.g., the label probabilities output by a deep neural network). The threshold $\tau$ is computed by solving the following optimization problem:
\begin{equation}
\hat\tau = \operatorname*{\arg\max}_{\tau \geq 0}~ \tau ~~ \text{subj. to} \sum_{(x, y) \in Z} \mathbb{I}[y \notin C_\tau(x)] \leq k^*,
\label{eqn:algorithm}    
\end{equation}
where
\begin{equation*}
k^* =\operatorname*{\arg\max}_{k\in\mathbb{N}\cup\{0\}}~k\qquad\text{subj. to}\qquad F(k;N,\alpha) \leq \delta,    
\end{equation*}
where $F(k; N, \alpha)$ is the cumulative distribution function of the binomial random variable $\text{Binomial}(N, \alpha)$
with $N$ trials and success probability $\alpha$. 
Maximizing $\tau$ corresponds to minimizing the prediction set size. We have the following theorem:

\begin{theorem}[\cite{vovk2012conditional,park2021pac}] \label{thm:pred_set}
$C_{\hat\tau}$ is $(\alpha, \delta)$-correct for $\hat\tau$ as in \eqref{eqn:algorithm}.
\end{theorem}

\subsection{Conformal Prediction and PAC Prediction Set Comparison}
\label{sec:cp_pac}

\paragraph{Conformal Prediction Guarantee}

Formally, we can write the conformal prediction guarantee as $$Pr_{(X,Y) \sim \mathcal{D}} (Y \in C(X)) \geq 1 - \alpha.$$

In other words, the prediction sets constructed by conformal prediction guarantee that over the whole distribution $\mathcal{D}$, the probability that the true label is contained in the set is at least $1-\alpha$. Note that this coverage probability is marginalized over all possible calibration sets. On the other hand, for a specific calibration set $B$, this guarantee might not hold. For example, the guarantee will not hold if the samples in $B$ are concentrated in a small region of the joint distribution and therefore are not representative of the joint distribution $\mathcal{D}$. 

\paragraph{PAC Prediction Set Guarantee}

Formally, we can write the guarantee of the PAC prediction set guarantee as
$$
\Pr_{B \sim \mathcal{D}^N} (Pr_{(X, Y) \sim \mathcal{D}} \geq 1-\alpha) \geq 1-\delta.
$$

Compared to the conformal prediction guarantee, the difference is the outer probability, which is on the given calibration set $B$. Intuitively, the guarantee of the PAC prediction set says that conditioning on the given calibration set $B$, we can say with high confidence (at least $1-\delta$) that the true label is contained in the constructed set $C(X)$ with high probability $(1-\alpha)$. As a result, the PAC prediction set guarantee is stronger than the conformal prediction guarantee, as the PAC prediction set guarantee is over an individual calibration set, while the conformal prediction guarantee is marginalized over all possible calibration sets.

\subsection{Bayesian Optimization}
\label{sec:bayesian}

Bayesian optimization (BO) is a technique to find the global optimum of a potentially nonconvex, nonlinear, or nonclosed-form objective function $f$ with decision variables $\{b^1, \ldots, b^M\}$. It builds a probabilistic model of the objective function and then selects parameters that could maximize it. The model is then refined using the chosen parameters. This process is repeated until an iteration budget $T$ is reached, as shown in Algorithm~\ref{alg:bo}~\cite{https://doi.org/10.48550/arxiv.1807.02811}. Our implementation of Bayesian optimization is based on \textit{scikit-optimization}~\cite{head_2021_5565057}.

\begin{algorithm}[tbh]
\caption{Bayesian Optimization}
\label{alg:bayopt}
\begin{algorithmic}[1]
\State Place a Gaussian process prior $g$ on $f$.
\State Observe $f$ at $t_0$ points according to an initial space-filling experimental design. Set $t=t_0$.
\While{$t \leq T$}
\State Update the posterior probability distribution on $g$ using all available data.
\State Let $b_t$ be a maximizer of the acquisition function over $b$, where the acquisition function is computed using the current posterior distribution.
\State Observe $f(b_t)$.
\State Increment $t$.
\EndWhile
\State {\bfseries Return:} either the point evaluated with the smallest $f(b)$ or the point with the smallest posterior.
\end{algorithmic}
\label{alg:bo}
\end{algorithm}

\section{Proofs}
\label{sec:e2e_proof}

\begin{proof}[Proof of Lemma~\ref{co:ret}]
First, based on Assumption~\ref{as:iid}, samples collected for the construction of the retrieval prediction set are i.i.d. with unobserved samples, satisfying the i.i.d. (exchangeability) assumption required by conformal prediction (PAC prediction set).

Second, based on Assumption~\ref{as:ret}, for each input question $q$, since its relevant passage can be retrieved, the prediction set can contain the relevant passage if the threshold $\tau_{\text{Ret}}$ is appropriately set. (Otherwise, the prediction set cannot contain the relevant passage even if all retrieved passages are included.)

Third, since we construct the retriever set following conformal prediction with the error level being $\alpha_{\text{Ret}}$, the resulting retriever sets satisfy:
\begin{equation*}
\Pr_{(q, p^*) \sim \mathcal{D_{\text{Passage}}}} (p^* \in C_{\text{Ret}}(q)) \geq 1-\alpha_{\text{Ret}}.
\end{equation*}
\label{pf:ret}
\end{proof}

\begin{proof}[Proof of Lemma~\ref{co:chat}]
First, based on Assumption~\ref{as:iid}, samples collected for the construction of the LLM prediction set are i.i.d. with unobserved samples, satisfying the i.i.d. (exchangeability) assumption required by conformal prediction (PAC prediction set).

Second, based on Assumption~\ref{as:chat}, for every input question and its most relevant passage $q^*$, since its semantically correct responses can be retrieved, the prediction set can contain correct responses if the threshold $\tau_{\text{LLM}}$ is appropriately set. (Otherwise, the prediction set cannot contain correct responses even if all responses are included.)

Third, since we construct the LLM prediction set following conformal prediction with the error level being $\alpha_{\text{LLM}}$, the resulting retriever sets satisfy:
\begin{equation*}
\Pr_{(q, p^*, r^*) \sim \mathcal{D_{\text{Response}}}} (r^* \in C_{\text{LLM}}(q, p^*) ) \geq 1-\alpha_{\text{LLM}}.
\end{equation*}
\label{pf:chat}
\end{proof}


\begin{proof}[Proof of Theorem~\ref{th:e2e}]
    We prove this theorem by union bound. Specifically, given two event $A$ and $B$, we have the following inequality:
\begin{equation*}
\Pr(A \cup B) = \Pr(A) + \Pr(B) - \Pr(A \cap B) \leq \Pr(A) + \Pr(B).
\end{equation*}
In TRAQ, let event $A$ be
\begin{equation*}
    \{p^* \notin C_{\rm Ret}(q)\};
\end{equation*}
\noindent and event $B$ be 
\begin{equation*}
    \{r^* \notin C_{\rm LLM}(q, p^*)\}.
\end{equation*}
By Lemma~\ref{co:ret} and~\ref{co:chat}, we have
\begin{align*}
    \Pr (p^* \notin C_{\rm Ret}(q)) &= 1 - \Pr (p^* \in C_{\rm Ret}(q)) \leq \alpha_{\rm Ret} \\
    \Pr (p^* \notin C_{\rm LLM}(q, p^*)) &= 1 - \Pr (r^* \in C_{\rm LLM}(q, p^*)) \leq \alpha_{\rm LLM}.
\end{align*}

Then, we have the following inequalities
\begin{align*}
    \Pr & (r^* \notin C_{\text{Agg}}(q)) \\
    &= \Pr ( r^* \notin \cup_{p \in C_{\rm Ret}(q)} C_{\rm LLM}(q, p) )
    \\
    &= \Pr ( r^* \notin \cup_{p \in C_{\rm Ret}(q)} C_{\rm LLM}(q, p), A ) + \Pr ( r^* \notin \cup_{p \in C_{\rm Ret}(q)} C_{\rm LLM}(q, p), A^C )
    \\
    &= \Pr ( r^* \notin \cup_{p \in C_{\rm Ret}(q)} C_{\rm LLM}(q, p) | A )\Pr(A) 
    + \Pr ( r^* \notin \cup_{p \in C_{\rm Ret}(q)} C_{\rm LLM}(q, p) | A^C ) \Pr(A^C)
    \\
    &\le
    \Pr(A) + \Pr ( r^* \notin \cup_{p \in C_{\rm Ret}(q)} C_{\rm LLM}(q, p) | A^C ) \Pr(A^C)
    \\
    &\le 
    \Pr(A) + \Pr ( r^* \notin C_{\rm LLM}(q, p^*))
    \\
    &\le \alpha_{\rm Ret} + \alpha_{\rm LLM} = \alpha.
\end{align*}
\label{pf:e2e}
\end{proof}

\subsection{PAC Prediction Set Construction}
\label{sec:pac_set}
To construct prediction sets with probably approximately correct (PAC) guarantees, we use the same nonconformity measures states in~\ref{sec:cp_sets} for retrieval and LLM tasks, respectively. Also, we will assign the error budgets $\alpha_{\text{Ret}}$ and $\alpha_{\text{LLM}}$ with $\alpha_{\text{Ret}}+\alpha_{\text{LLM}}=\alpha$. Additionally, we need to specify confidence levels for PAC prediction set. In our work, we specify $1-\frac{\delta}{2}$ to the retriever and LLM PAC prediction set. Then, we have the following Corollaries:

\begin{lemma}
Suppose the questions and their corresponding most relevant passage $p^*$'s are subject to the distribution $\mathcal{D_{\text{passage}}}$. Given the error budget $\alpha_{\text{Ret}}$ and confidence level $1-\frac{\delta}{2}$,, the constructed retriever sets satisfy the following inequality:
\begin{equation}
\Pr_{B \sim \mathcal{D_{\text{Passage}}}}[\Pr_{(q, p^*) \sim \mathcal{D_{\text{Passage}}}} (p^* \in C_{\text{Ret}}(q)) \geq 1-\alpha_{\text{Ret}}]\geq 1-\frac{\delta}{2}.
\label{eq:ret_pac}
\end{equation}
\label{co:ret_pac}
\end{lemma}

\begin{lemma}
Suppose the questions, their corresponding most relevant passage $p^*$'s, and semantically correct responses $r^*$ are subject to the distribution $\mathcal{D_{\text{Response}}}$. Given the error budget $\alpha_{\text{LLM}}$ and confidence level $1-\frac{\delta}{2}$, if Assumption~\ref{as:iid} and Assumption~\ref{as:chat} hold, the LLM sets using PAC prediction set satisfy the following inequality:
\begin{equation}
\Pr_{B \sim \mathcal{D_{\text{Response}}}^N}[\Pr_{(q, p^*, r^*) \sim \mathcal{D_{\text{Response}}}} (r^* \in C_{\text{LLM}}(q, p^*)) \geq 1-\alpha_{\text{LLM}}] \geq 1-\frac{\delta}{2}.
\label{eq:chat_pac}
\end{equation}
\label{co:chat_pac}
\end{lemma}

\begin{theorem}
Suppose the questions $q$'s, and semantically correct responses $r^*$'s are subject to the distribution $\mathcal{D}$; a user-specified error level $\alpha$ is given. By aggregating retriever sets with error budget $\alpha_{\text{Ret}}$ with LLM sets with error budget $\alpha_{\text{LLM}}$ and confidence levels $1-\delta/2$, with $\alpha=\alpha_{\text{Ret}} + \alpha_{\text{LLM}}$, the aggregation sets satisfy the following inequality:
$$\Pr_{B \sim \mathcal{D}}[\Pr_{(q, r^*) \sim \mathcal{D}} (r^* \in C_{\text{Agg}}(q)) \geq 1-\alpha] \geq 1-\delta.$$
\label{th:e2e_pac}
\end{theorem}

\begin{proof}[Proof of Theorem~\ref{th:e2e_pac}]
Given Lemmas~\ref{co:ret_pac} \&~\ref{co:chat_pac} and $\alpha_{\text{Ret}}+\alpha_{\text{LLM}}=\alpha$, we can prove the end-to-end guarantee in the following way:
the $1-\alpha$ coverage guarantee can be proved as the proof of Theorem~\ref{th:e2e}. The confidence bound holds ($1-\delta$) by taking a union bound over the outer probabilities of Equation~\eqref{eq:ret_pac} and~\eqref{eq:chat_pac}.
\label{pf:e2e_pac}
\end{proof}

\section{Additional Results}
\subsection{Individual Coverage}

\begin{figure}[H] 
\begin{subfigure}{0.32\textwidth}
\includegraphics[width=\linewidth]{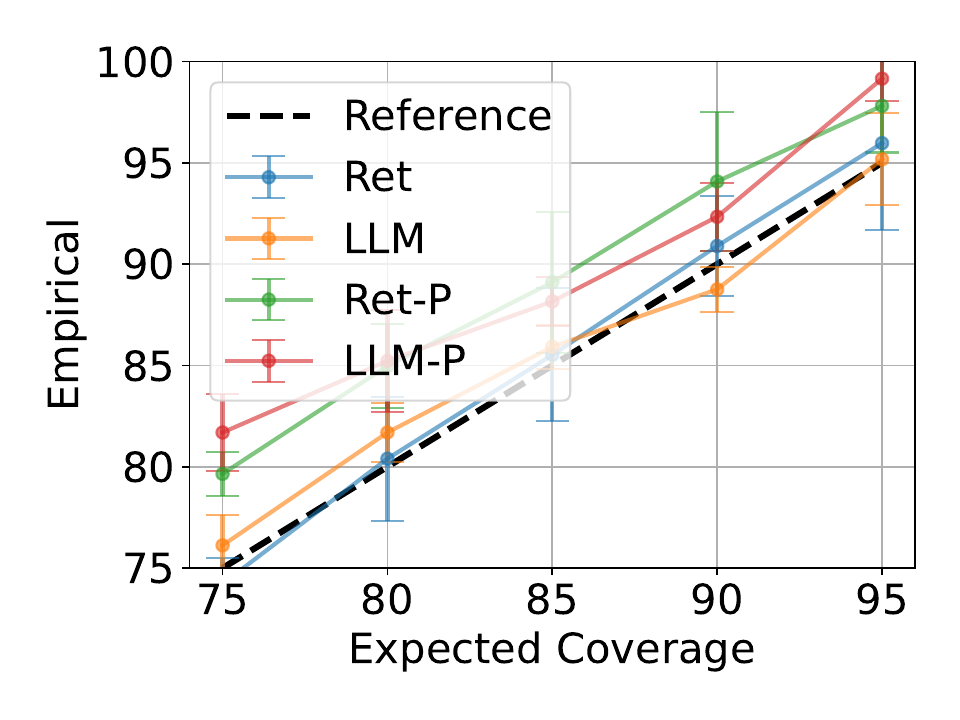}
\caption{Natural Question} \label{fig:a_individual}
\end{subfigure}\hspace*{\fill}
\begin{subfigure}{0.32\textwidth}
\includegraphics[width=\linewidth]{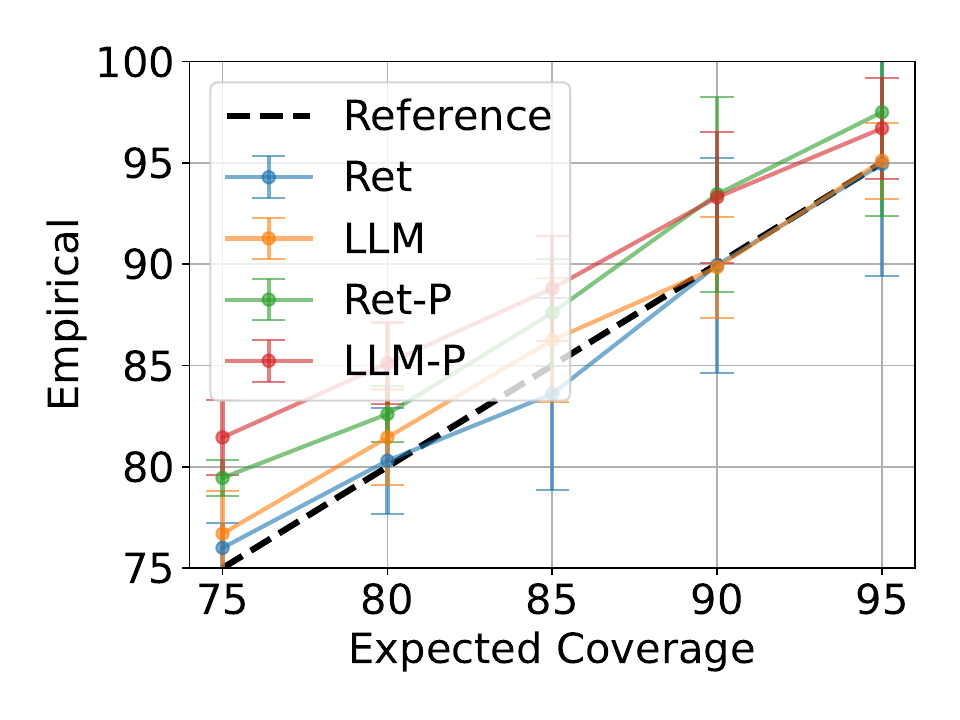}
\caption{TriviaQA} \label{fig:b_individual}
\end{subfigure}\hspace*{\fill}
\begin{subfigure}{0.32\textwidth}
\includegraphics[width=\linewidth]{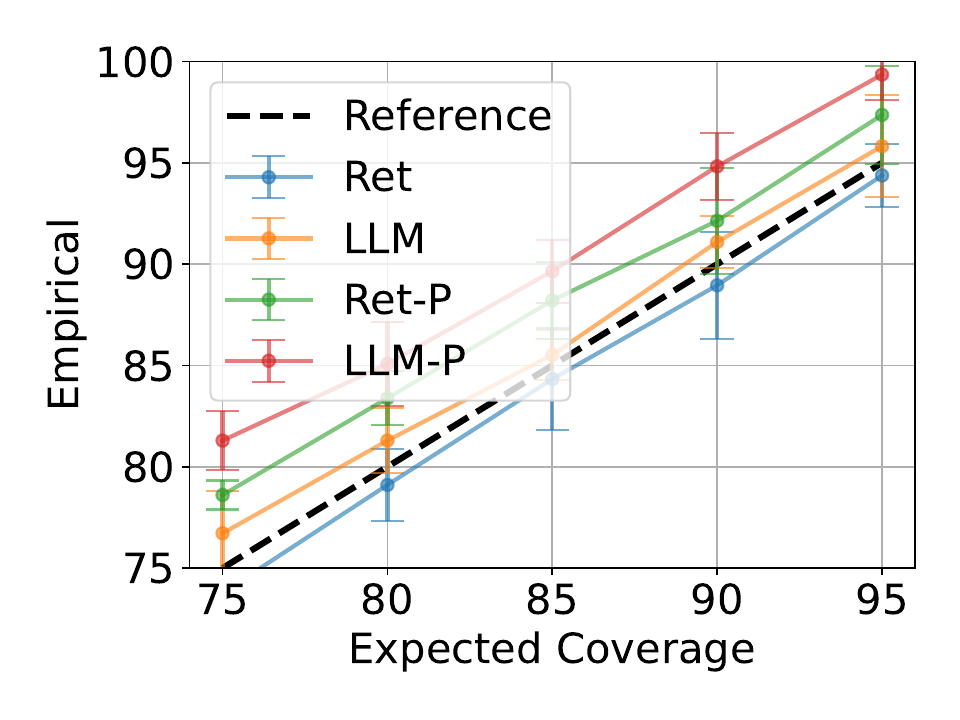}
\caption{SQuAD-1} \label{fig:c_individual}
\end{subfigure}

\medskip
\begin{subfigure}{0.32\textwidth}
\includegraphics[width=\linewidth]{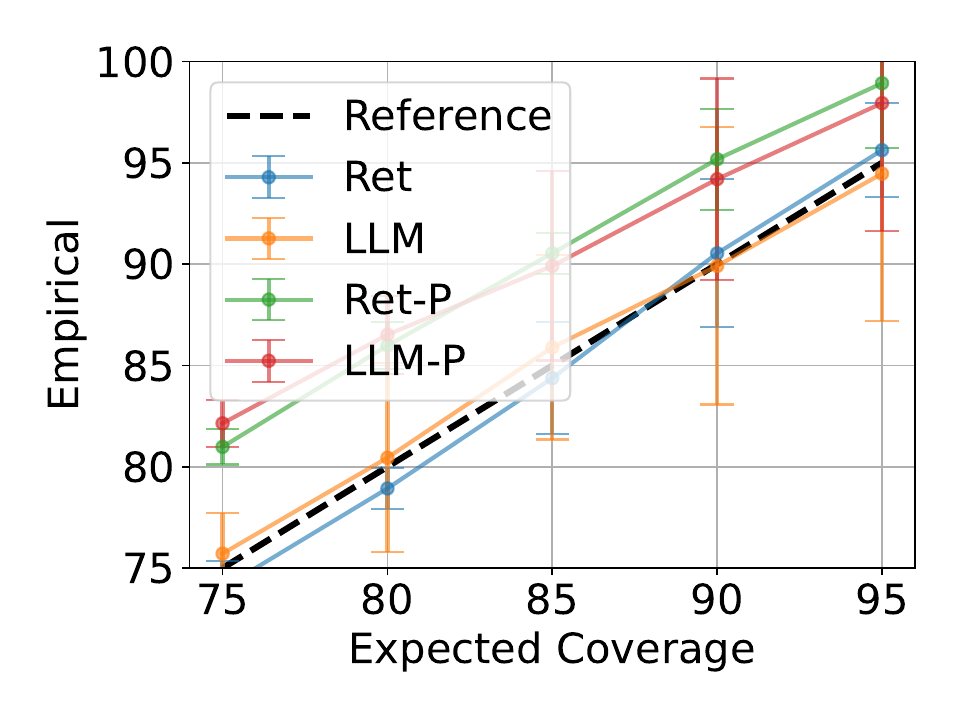}
\caption{Natural Question} \label{fig:d_individual}
\end{subfigure}\hspace*{\fill}
\begin{subfigure}{0.32\textwidth}
\includegraphics[width=\linewidth]{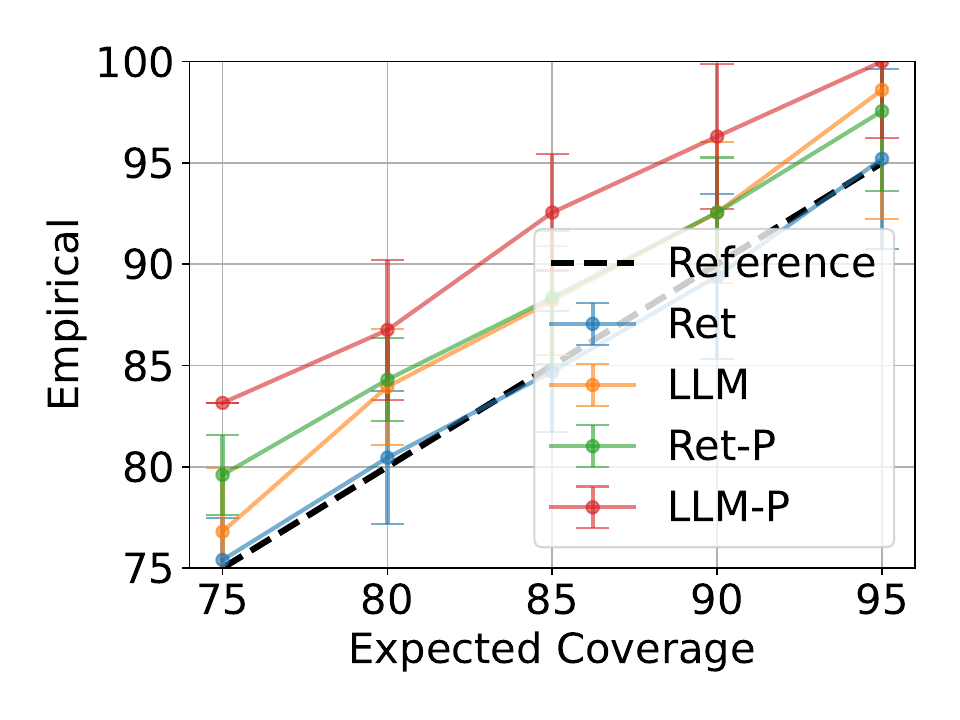}
\caption{TriviaQA} \label{fig:e_individual}
\end{subfigure}\hspace*{\fill}
\begin{subfigure}{0.32\textwidth}
\includegraphics[width=\linewidth]{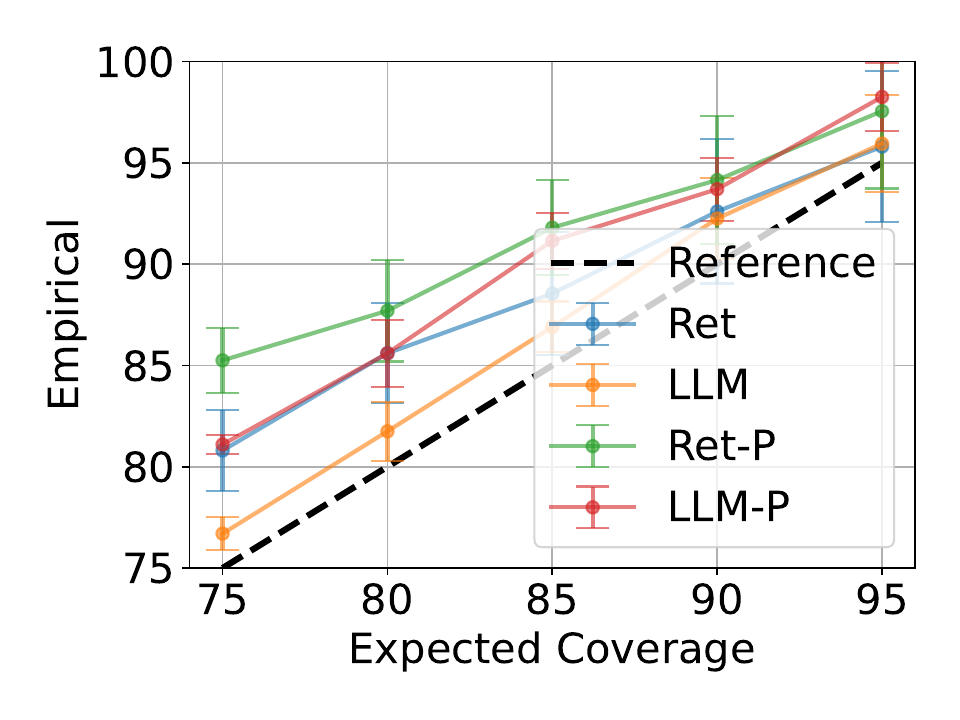}
\caption{SQuAD-1} \label{fig:f_individual}
\end{subfigure}

\caption{Individual coverages on all datasets using GPT-3.5 (first row) and Llama-2 (second row).} \label{fig:additional_individual}
\end{figure}

\subsection{Individual Coverage with More Random Seeds}

\begin{figure}[H] 
\begin{subfigure}{0.23\textwidth}
\includegraphics[width=\linewidth]{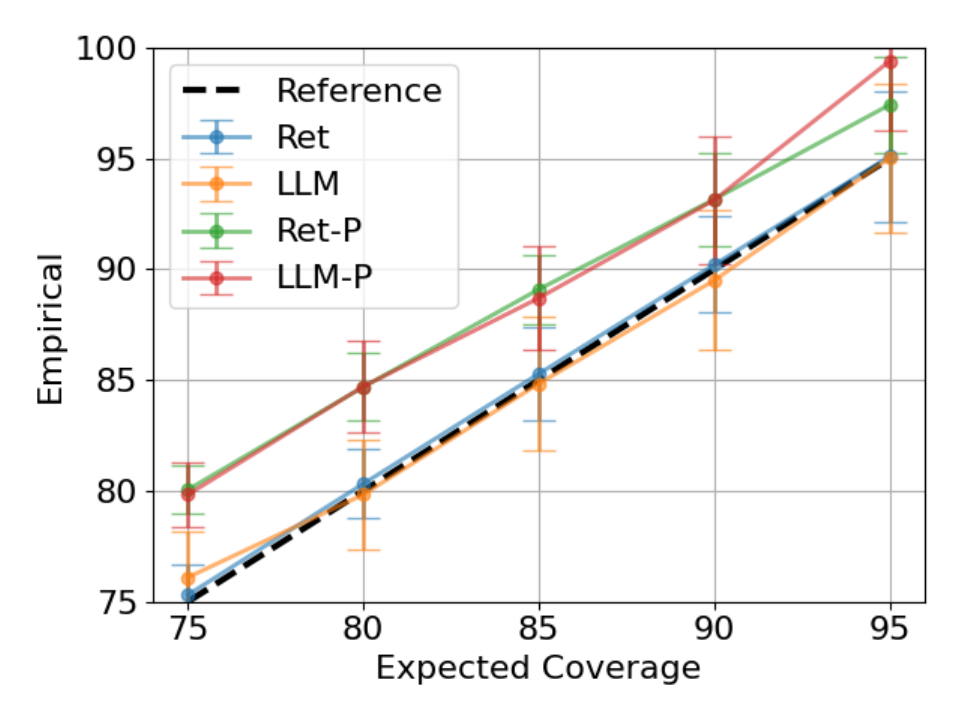}
\caption{BioASQ} \label{fig:a_more}
\end{subfigure}\hspace*{\fill}
\begin{subfigure}{0.23\textwidth}
\includegraphics[width=\linewidth]{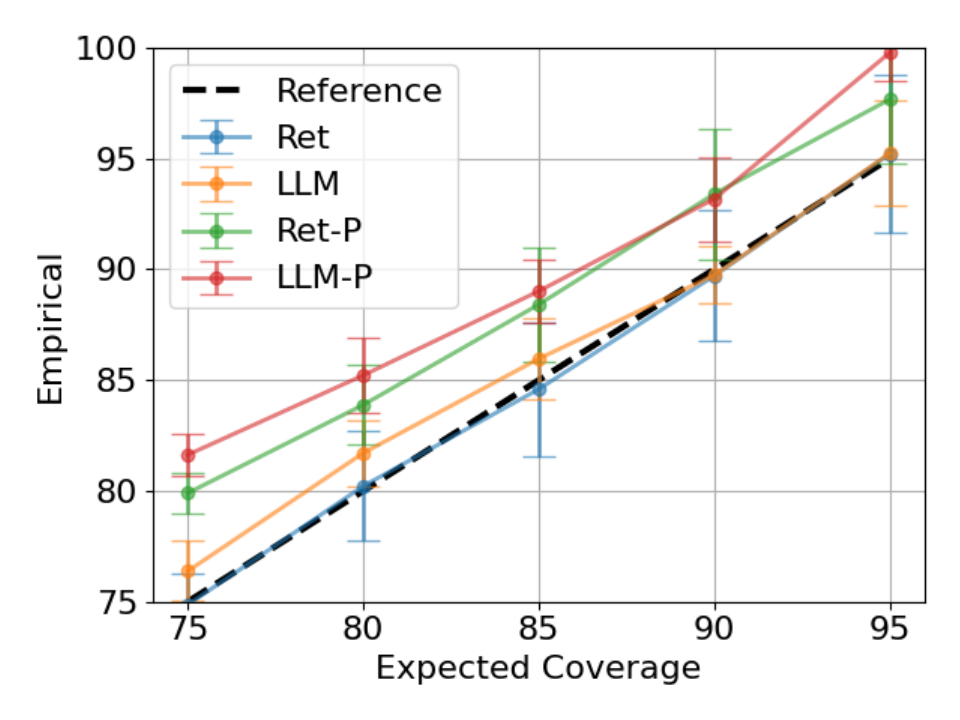}
\caption{Natural Question} \label{fig:b_more}
\end{subfigure}\hspace*{\fill}
\begin{subfigure}{0.23\textwidth}
\includegraphics[width=\linewidth]{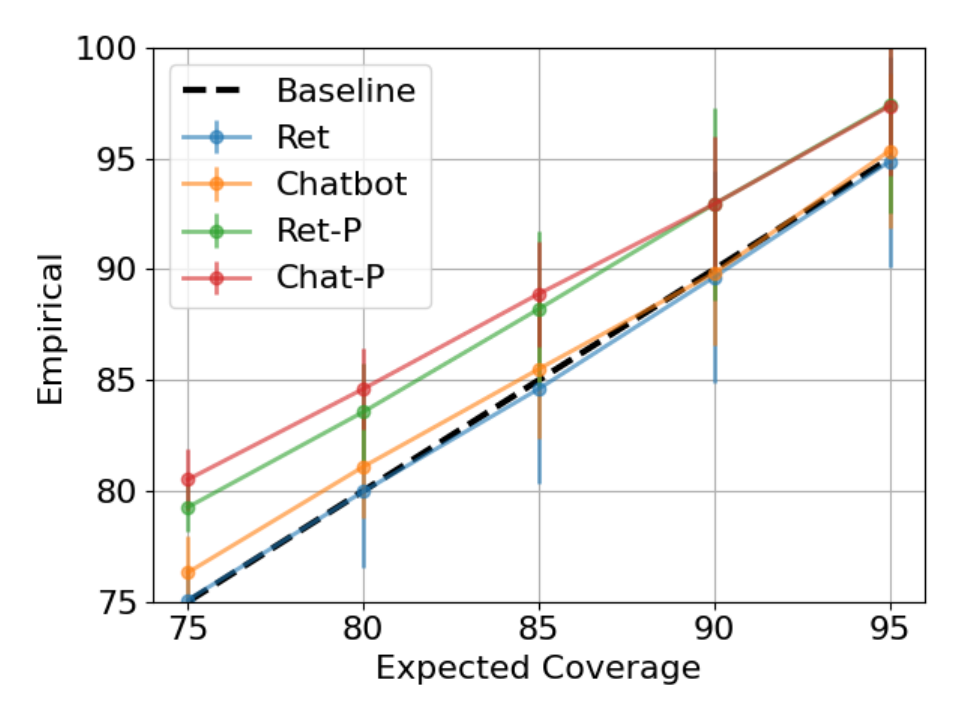}
\caption{TriviaQA} \label{fig:c_more}
\end{subfigure}\hspace*{\fill}
\begin{subfigure}{0.23\textwidth}
\includegraphics[width=\linewidth]{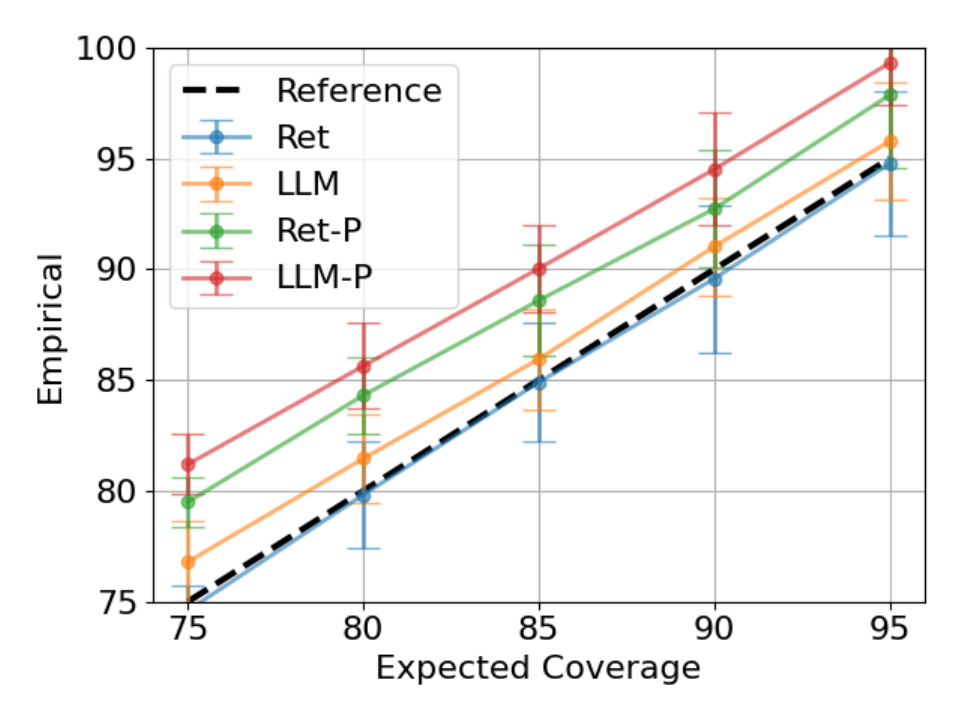}
\caption{SQuAD-1 using GPT-3.5} \label{fig:d_more}
\end{subfigure}

\medskip
\begin{subfigure}{0.23\textwidth}
\includegraphics[width=\linewidth]{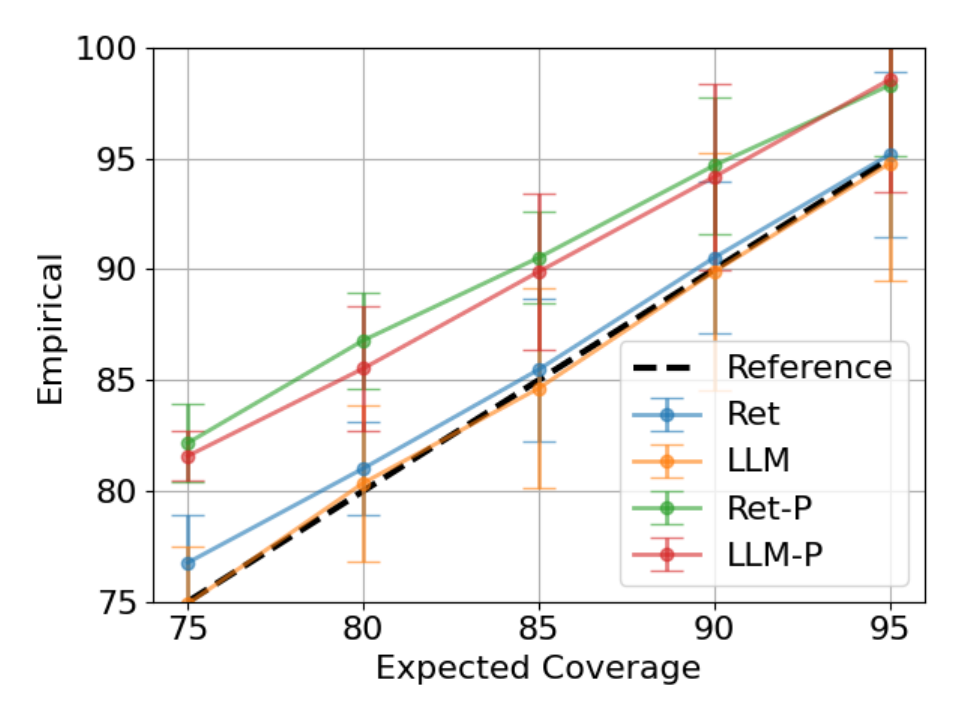}
\caption{Natural Question} \label{fig:e_more}
\end{subfigure}\hspace*{\fill}
\begin{subfigure}{0.23\textwidth}
\includegraphics[width=\linewidth]{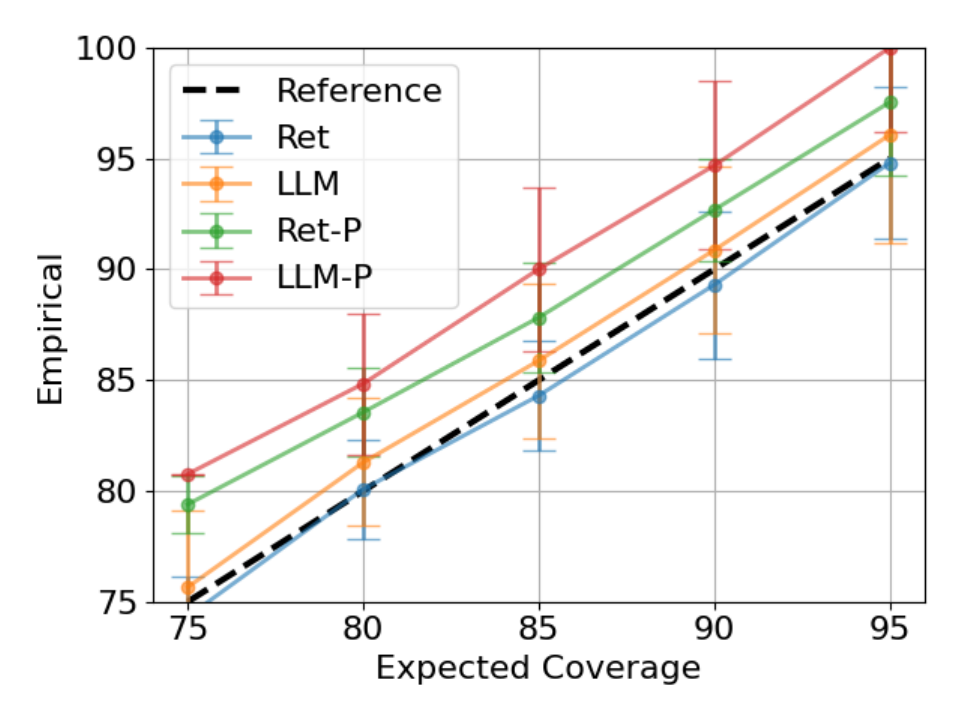}
\caption{TriviaQA} \label{fig:f_more}
\end{subfigure}\hspace*{\fill}
\begin{subfigure}{0.23\textwidth}
\includegraphics[width=\linewidth]{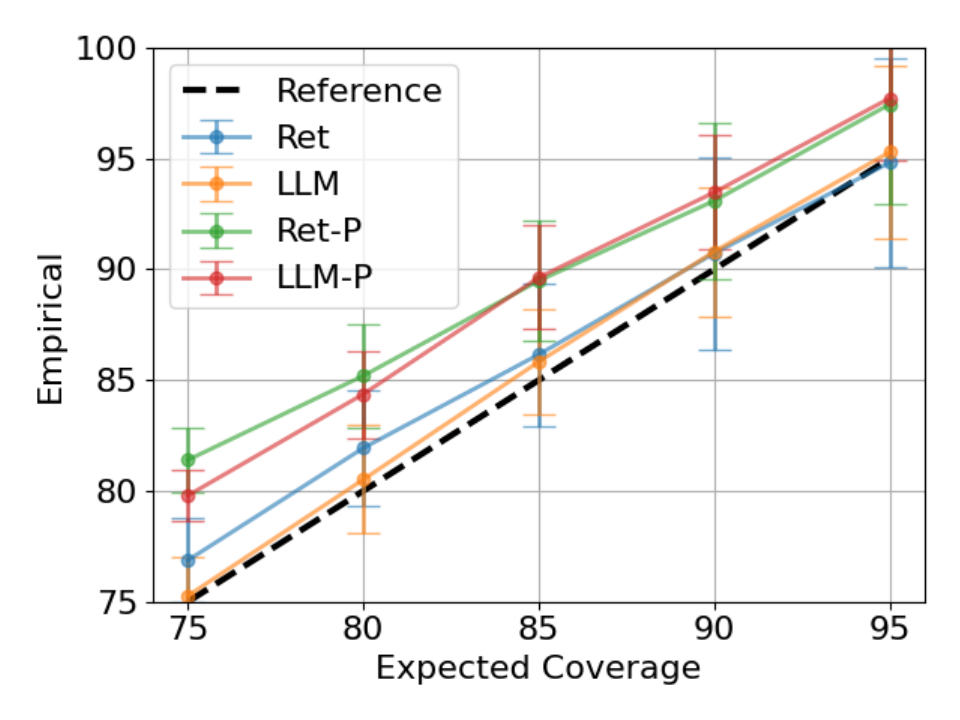}
\caption{SQuAD-1} \label{fig:g_more}
\end{subfigure}

\caption{Individual coverages on all Datasets using GPT-3.5 (first row) and Llama-2 (second row) with 20 random seeds.} \label{fig:additional_individual_20}
\end{figure}

\subsection{End-to-end Coverages}

\begin{figure}[H] 
\begin{subfigure}{0.32\textwidth}
\includegraphics[width=\linewidth]{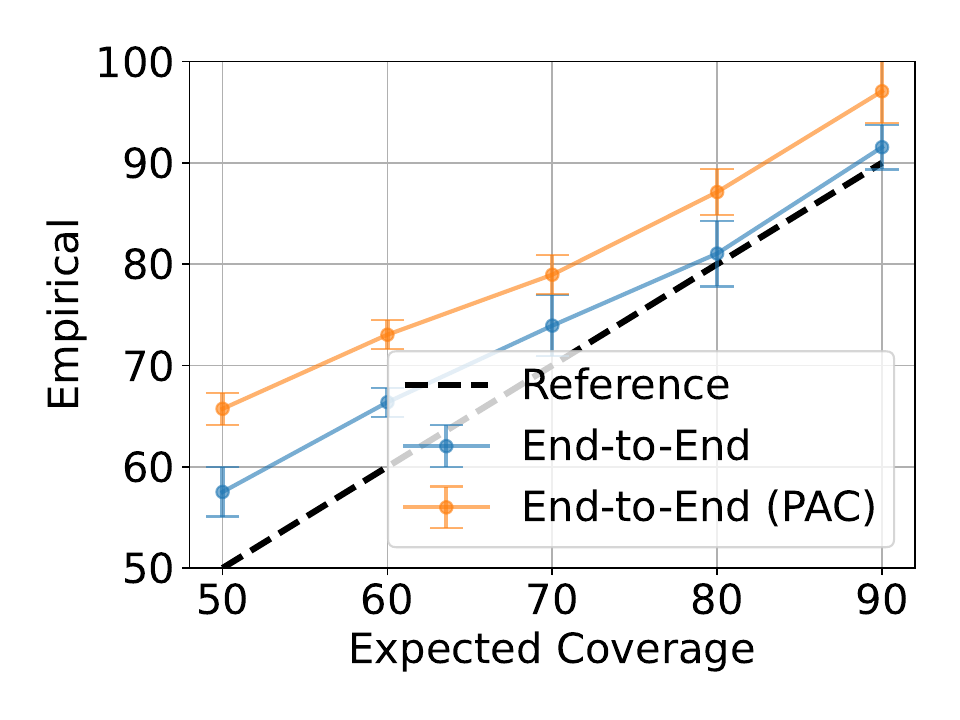}
\caption{Natural Question} \label{fig:e2e1_a}
\end{subfigure}\hspace*{\fill}
\begin{subfigure}{0.32\textwidth}
\includegraphics[width=\linewidth]{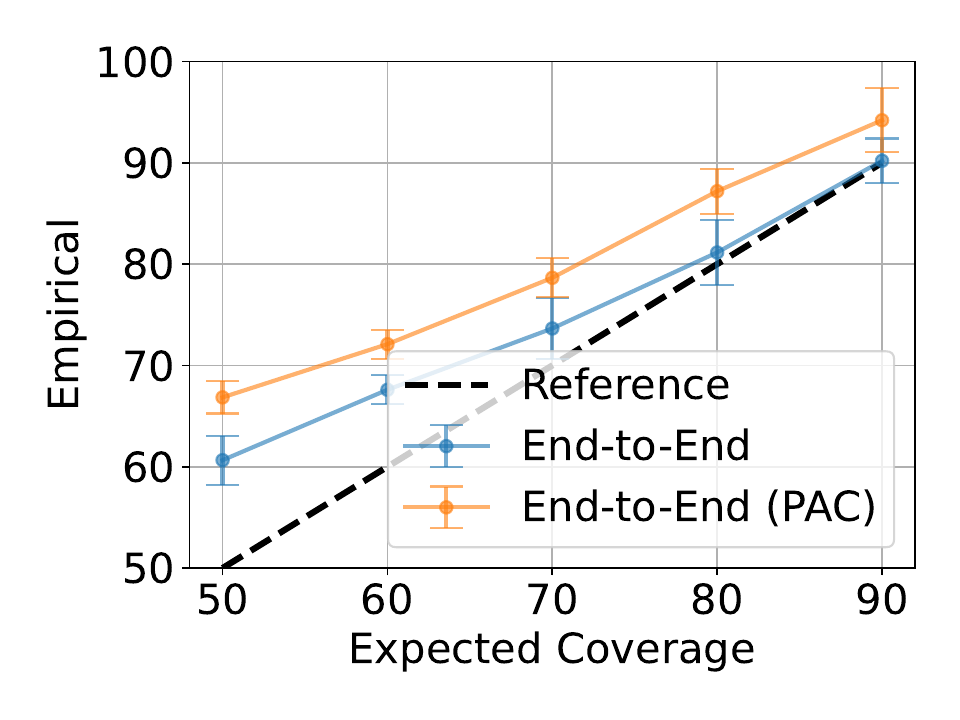}
\caption{TriviaQA} \label{fig:e2e1_b}
\end{subfigure}\hspace*{\fill}
\begin{subfigure}{0.32\textwidth}
\includegraphics[width=\linewidth]{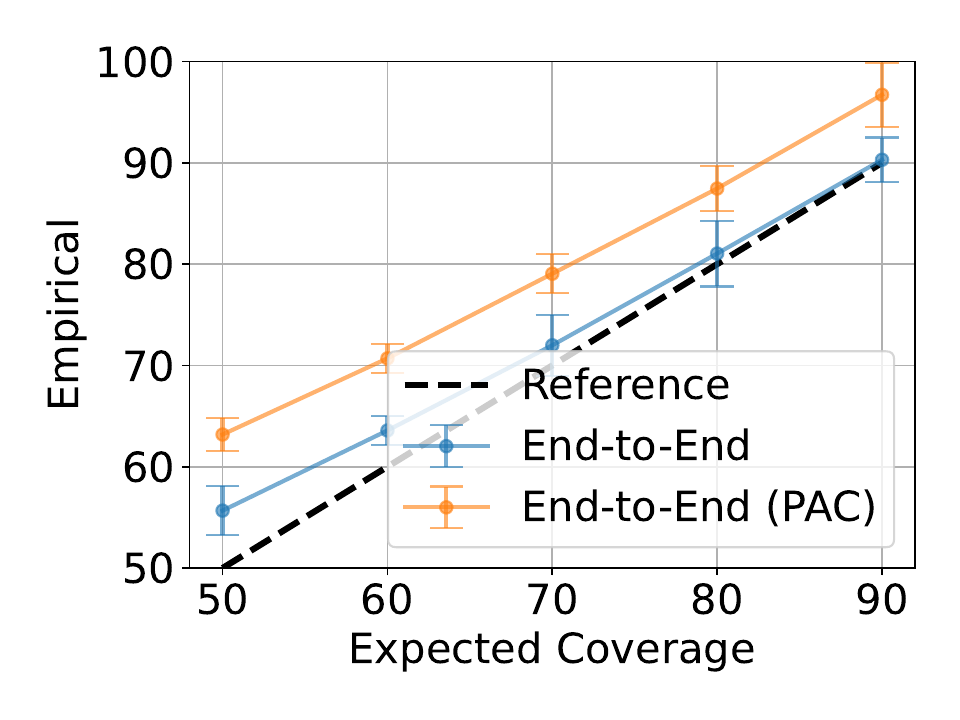}
\caption{SQuAD-1} \label{fig:e2e1_c}
\end{subfigure}

\medskip
\begin{subfigure}{0.32\textwidth}
\includegraphics[width=\linewidth]{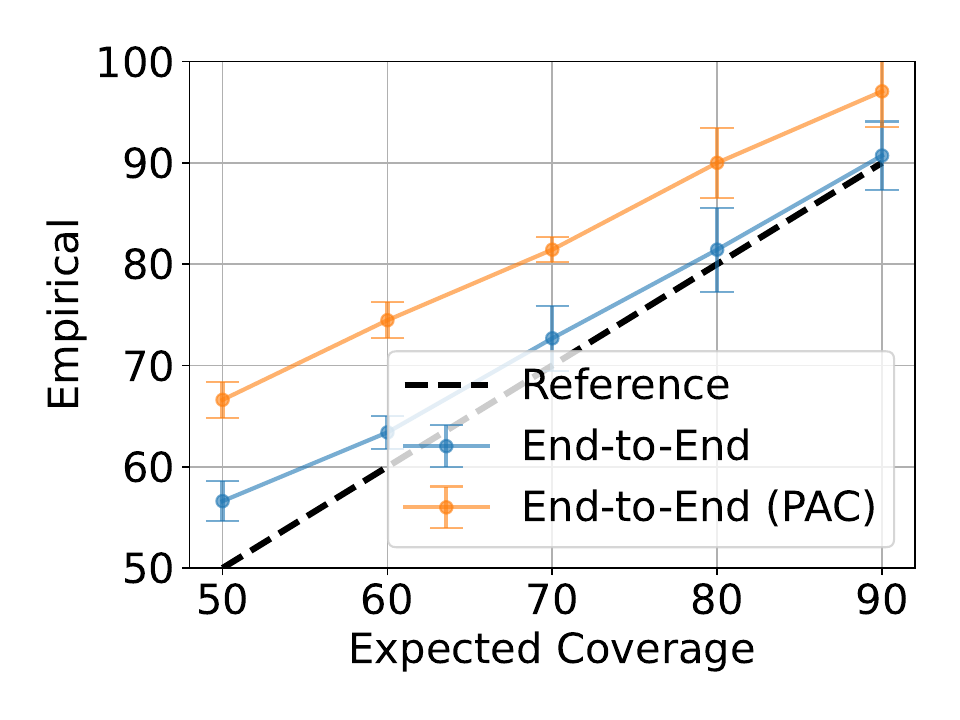}
\caption{Natural Question} \label{fig:e2e1_d}
\end{subfigure}\hspace*{\fill}
\begin{subfigure}{0.32\textwidth}
\includegraphics[width=\linewidth]{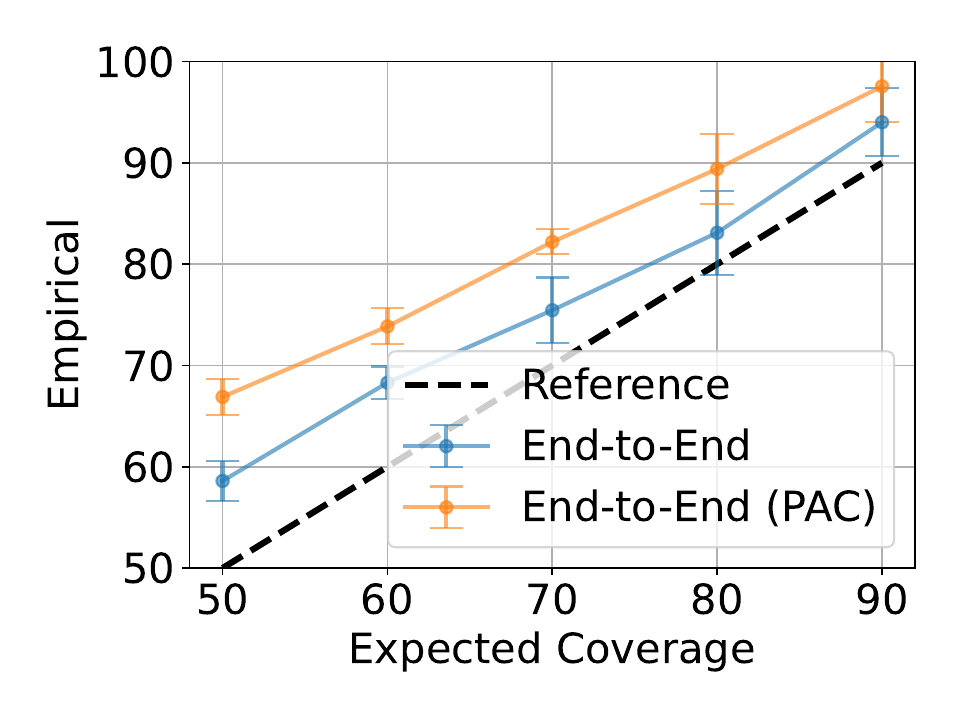}
\caption{TriviaQA} \label{fig:e2e1_e}
\end{subfigure}\hspace*{\fill}
\begin{subfigure}{0.32\textwidth}
\includegraphics[width=\linewidth]{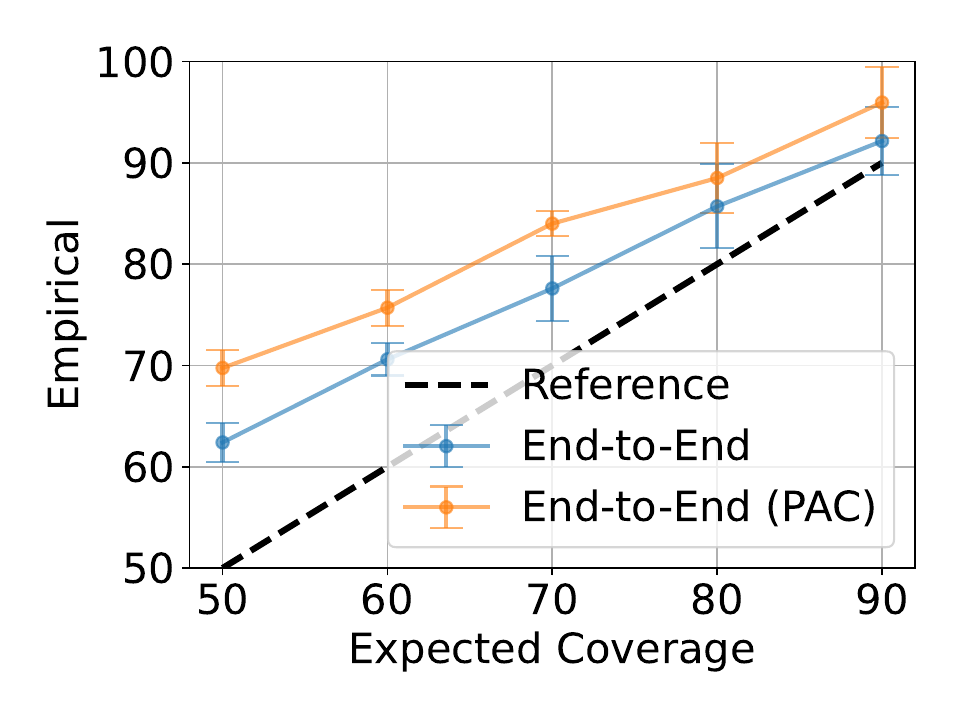}
\caption{SQuAD-1} \label{fig:e2e1_f}
\end{subfigure}

\caption{End-to-end coverage considering only the most relevant passage on all datasets using GPT-3.5 (first row) and Llama-2 (second row).} \label{fig:additional_e2e1}
\end{figure}

\subsection{End-to-end Coverages}

\begin{figure}[H] 
\begin{subfigure}{0.32\textwidth}
\includegraphics[width=\linewidth]{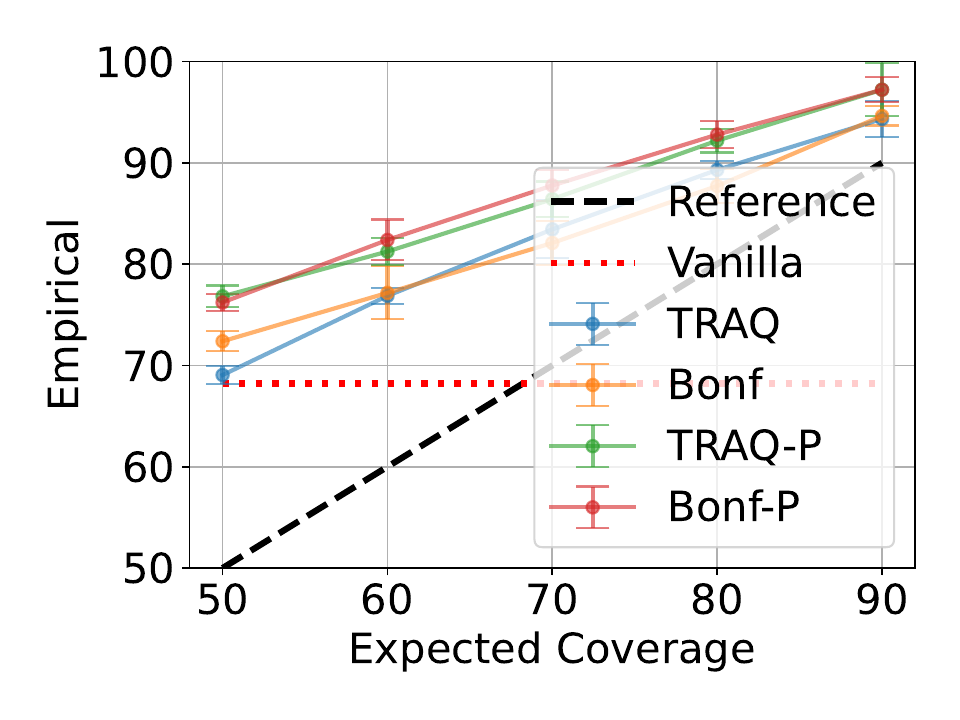}
\caption{Natural Question} \label{fig:e2e2_a}
\end{subfigure}\hspace*{\fill}
\begin{subfigure}{0.32\textwidth}
\includegraphics[width=\linewidth]{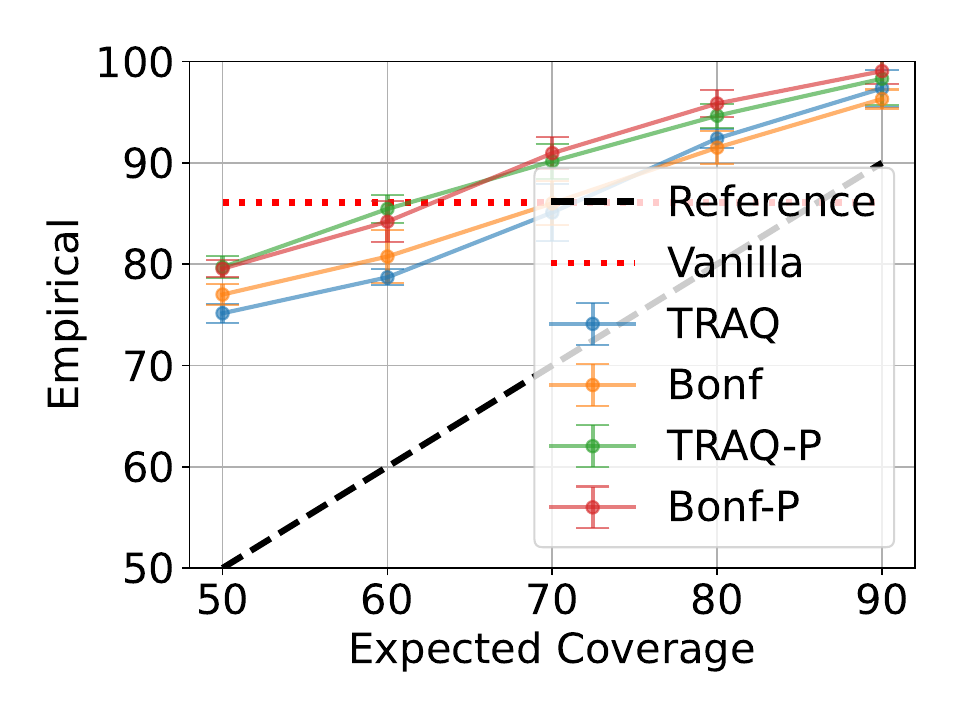}
\caption{TriviaQA} \label{fig:e2e2_b}
\end{subfigure}\hspace*{\fill}
\begin{subfigure}{0.32\textwidth}
\includegraphics[width=\linewidth]{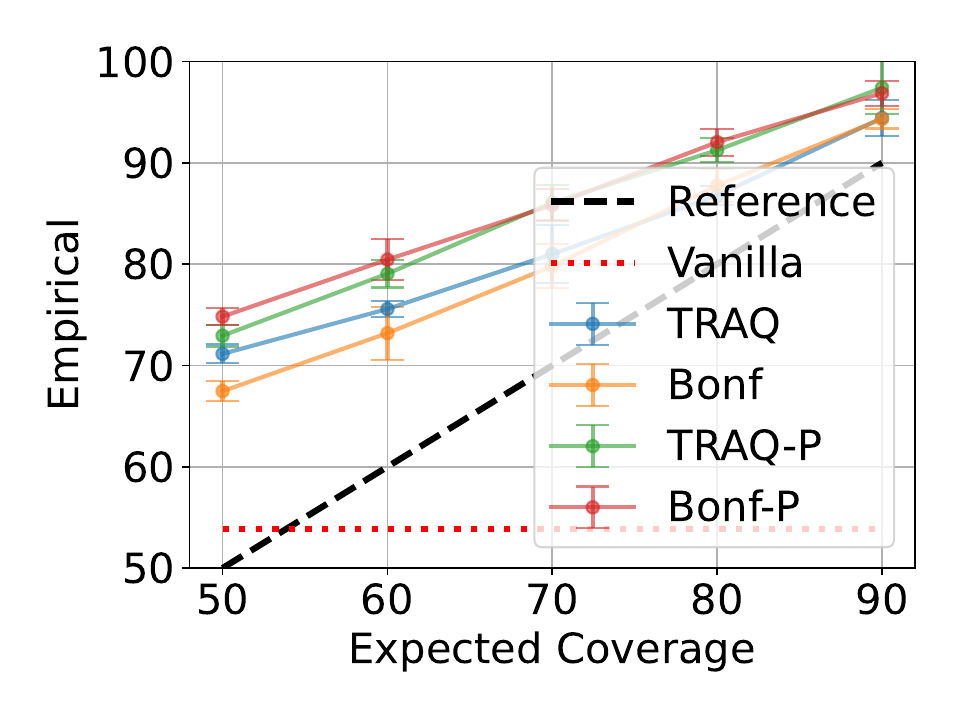}
\caption{SQuAD-1} \label{fig:e2e2_c}
\end{subfigure}

\medskip
\begin{subfigure}{0.32\textwidth}
\includegraphics[width=\linewidth]{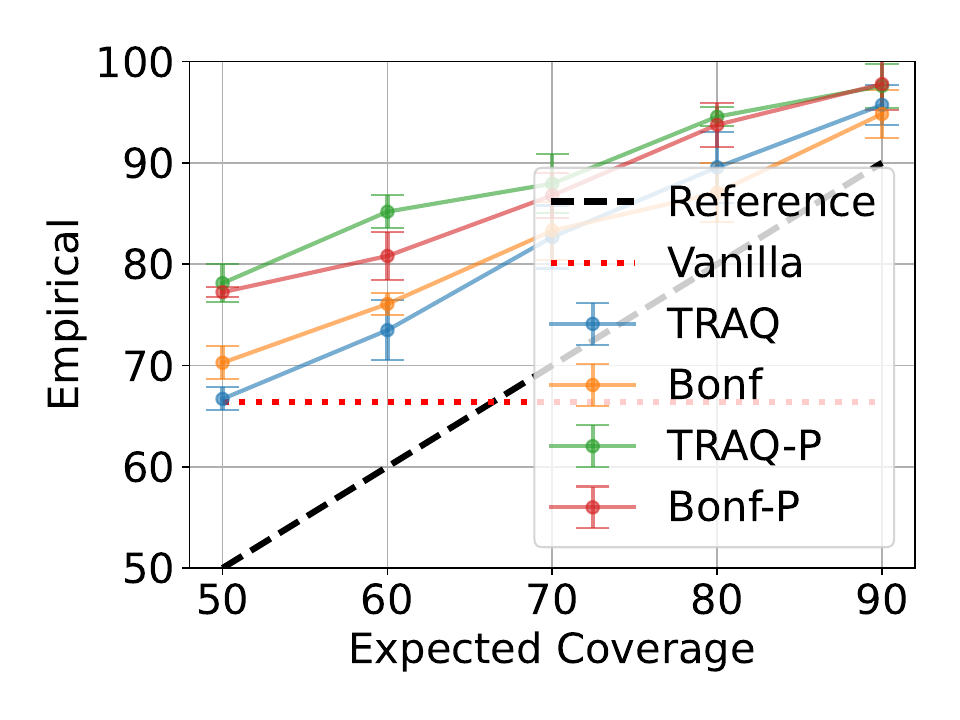}
\caption{Natural Question} \label{fig:e2e2_d}
\end{subfigure}\hspace*{\fill}
\begin{subfigure}{0.32\textwidth}
\includegraphics[width=\linewidth]{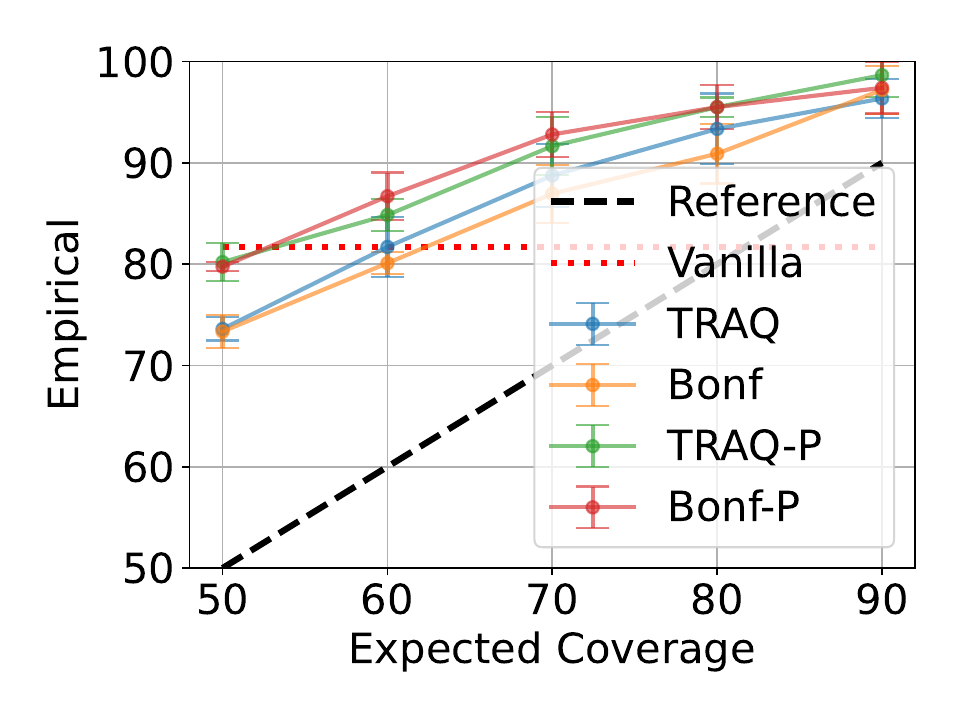}
\caption{TriviaQA} \label{fig:e2e2_e}
\end{subfigure}\hspace*{\fill}
\begin{subfigure}{0.32\textwidth}
\includegraphics[width=\linewidth]{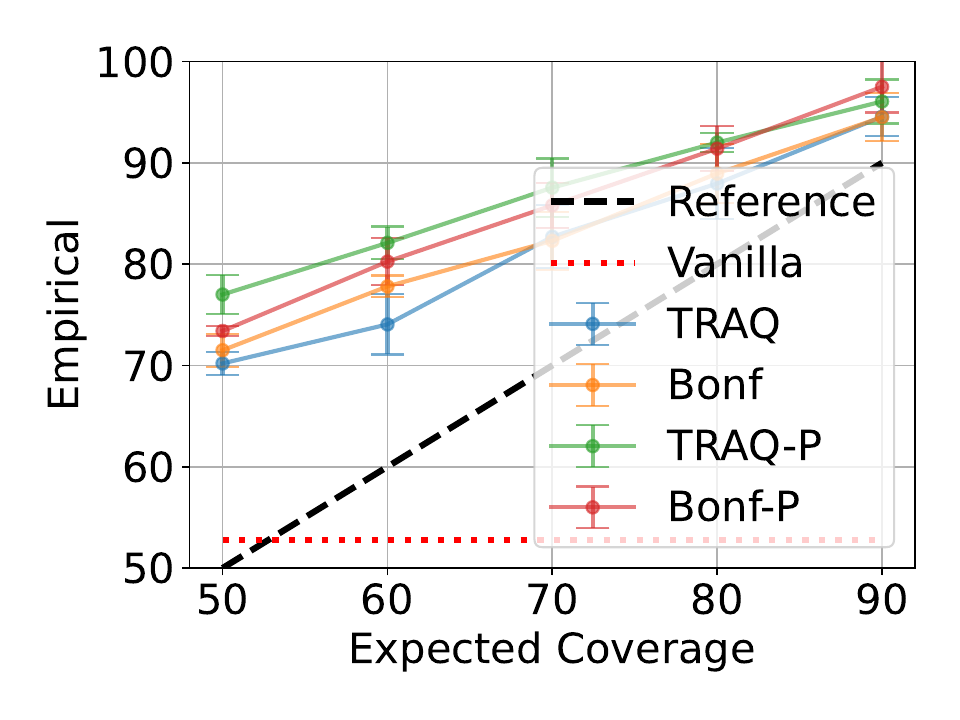}
\caption{SQuAD-1} \label{fig:e2e2_f}
\end{subfigure}

\caption{End-to-end coverage considering all passages on all datasets using GPT-3.5 (first row) and Llama-2 (second row).} \label{fig:additional_e2e2}
\end{figure}

\subsection{Performance}

Most of the results are similar to those in Figure~\ref{fig:avg_size}. The results on TriviaQA using Llama-2 have a relatively large prediction set size. This could be explained by the fact that the true scores on this task have a large variance. Therefore, the identified threshold $\tau_{\text{LLM}}$ was relatively low (as in Figure~\ref{fig:true_a} compared to other tasks (as in Figure~\ref{fig:true_b}).

\begin{figure}[H] 
\begin{subfigure}{0.32\textwidth}
\includegraphics[width=\linewidth]{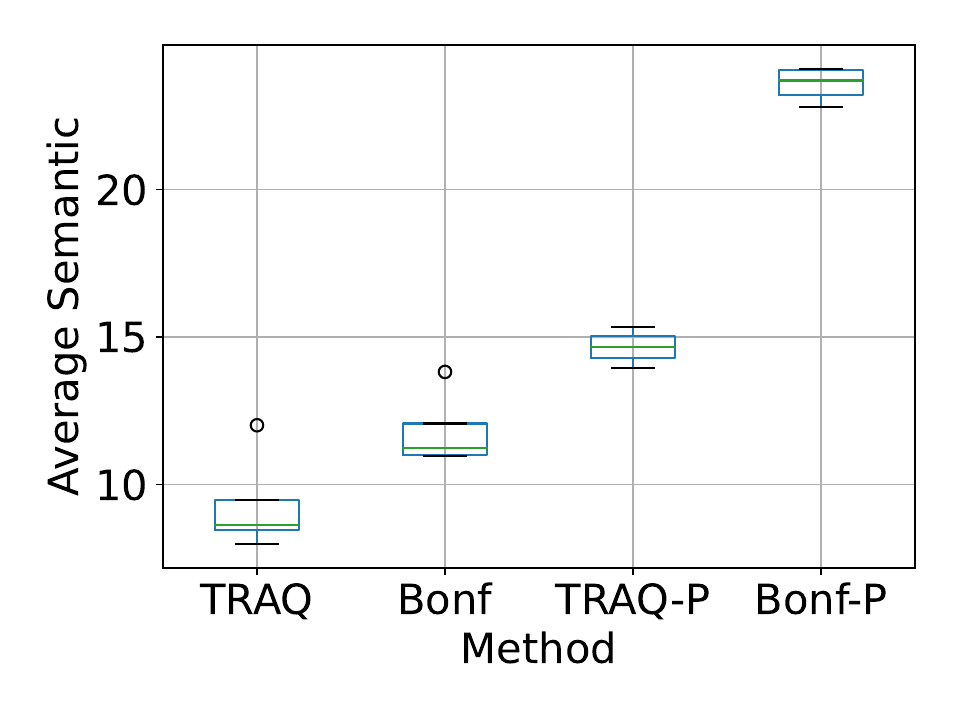}
\caption{Natural Question} \label{fig:eff_a}
\end{subfigure}\hspace*{\fill}
\begin{subfigure}{0.32\textwidth}
\includegraphics[width=\linewidth]{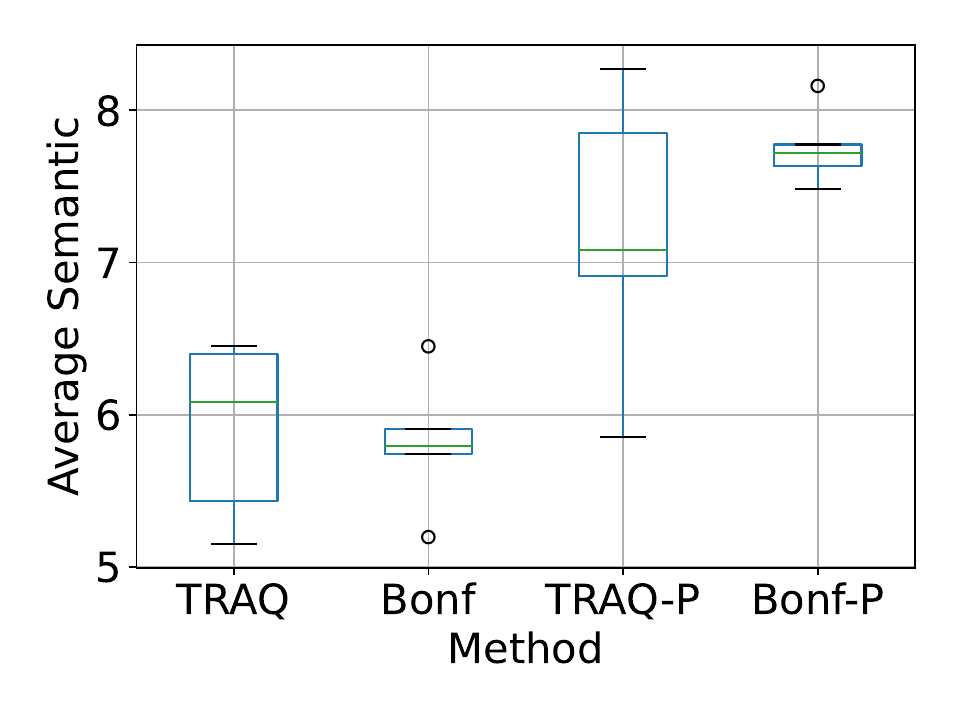}
\caption{TriviaQA} \label{fig:eff_b}
\end{subfigure}\hspace*{\fill}
\begin{subfigure}{0.32\textwidth}
\includegraphics[width=\linewidth]{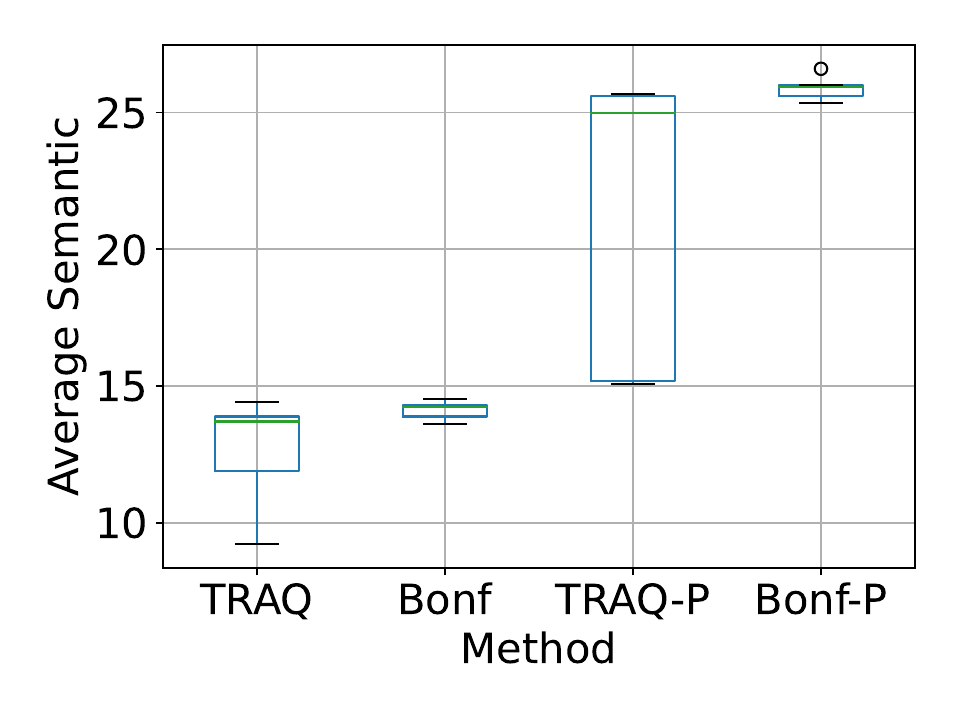}
\caption{SQuAD-1} \label{fig:eff_c}
\end{subfigure}

\medskip
\begin{subfigure}{0.32\textwidth}
\includegraphics[width=\linewidth]{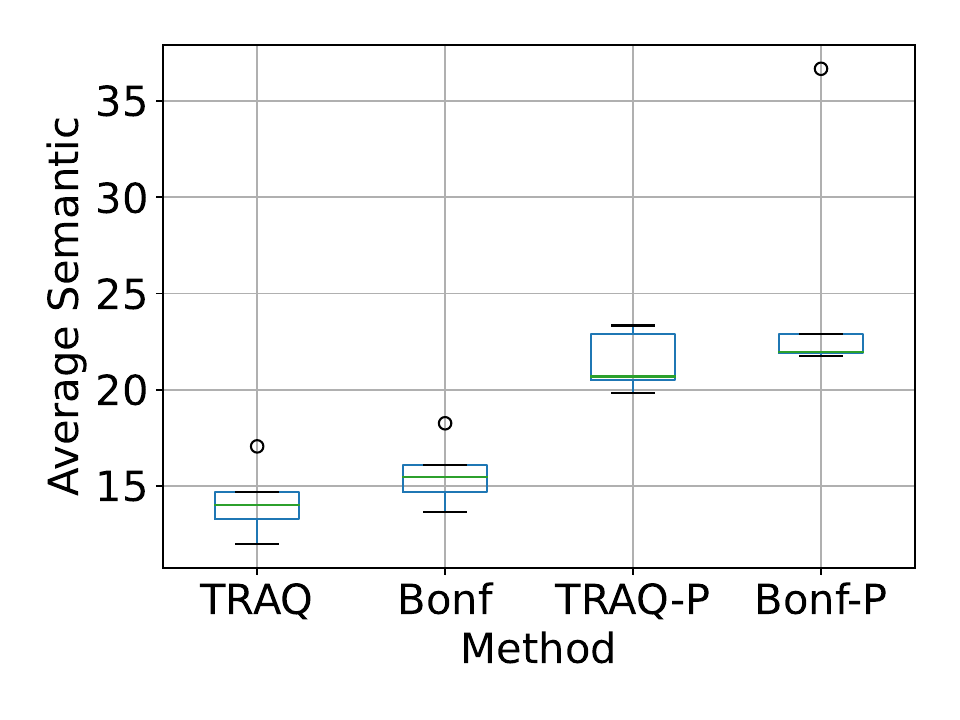}
\caption{Natural Question} \label{fig:eff_d}
\end{subfigure}\hspace*{\fill}
\begin{subfigure}{0.32\textwidth}
\includegraphics[width=\linewidth]{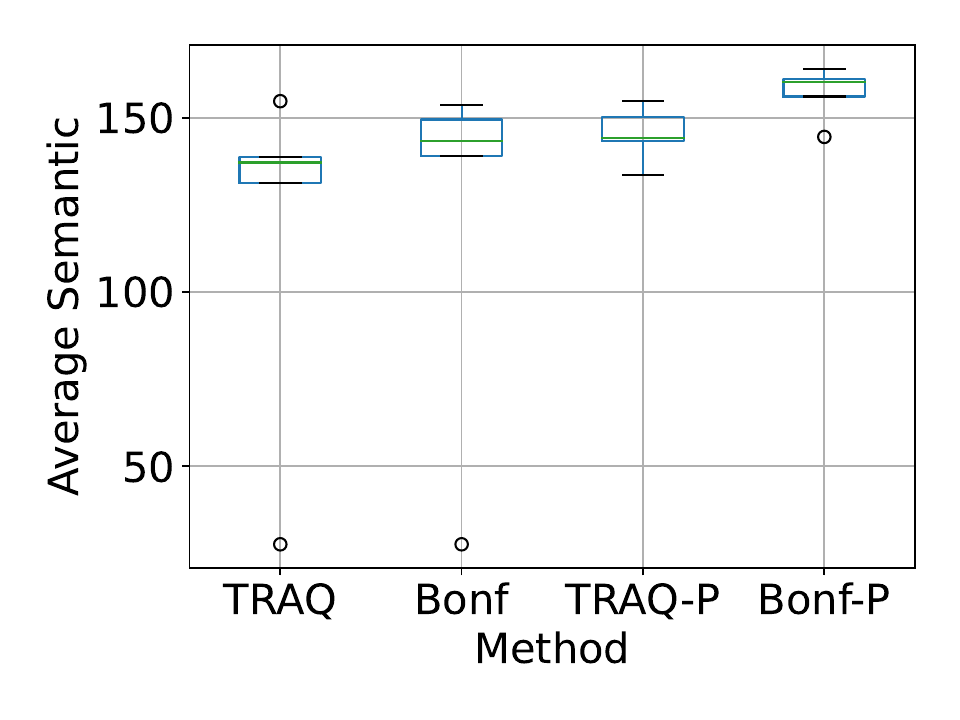}
\caption{TriviaQA} \label{fig:eff_e}
\end{subfigure}\hspace*{\fill}
\begin{subfigure}{0.32\textwidth}
\includegraphics[width=\linewidth]{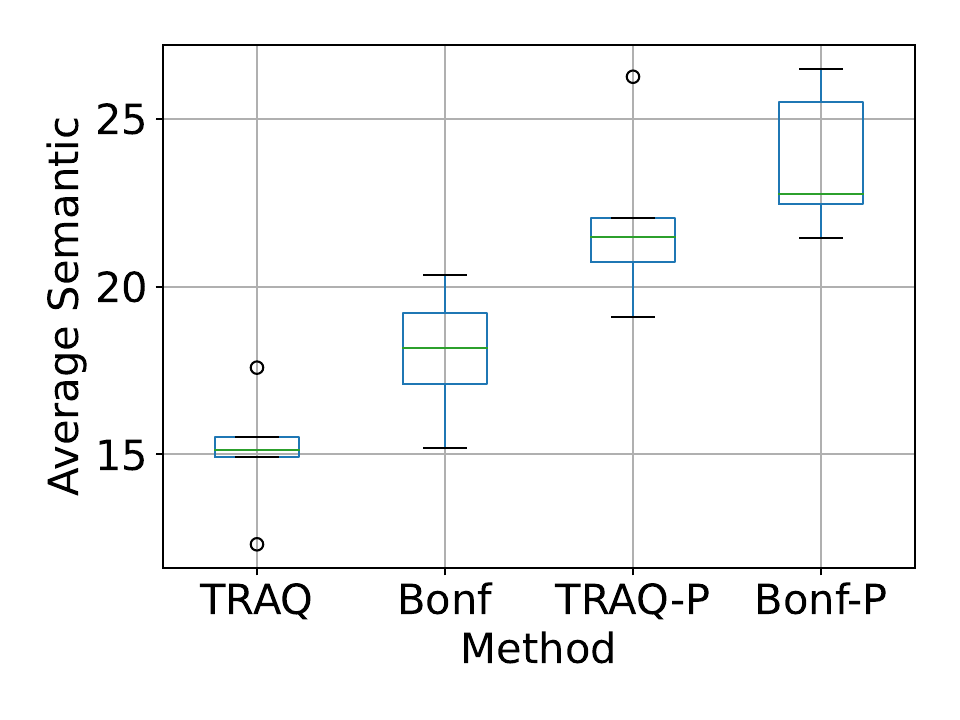}
\caption{SQuAD-1} \label{fig:eff_f}
\end{subfigure}

\caption{Average prediction set sizes on all datasets using GPT-3.5 (first row) and Llama-2 (second row).} \label{fig:additional_efficiency}
\end{figure}

\begin{figure}[H] 
\begin{subfigure}{0.45\textwidth}
\includegraphics[width=\linewidth]{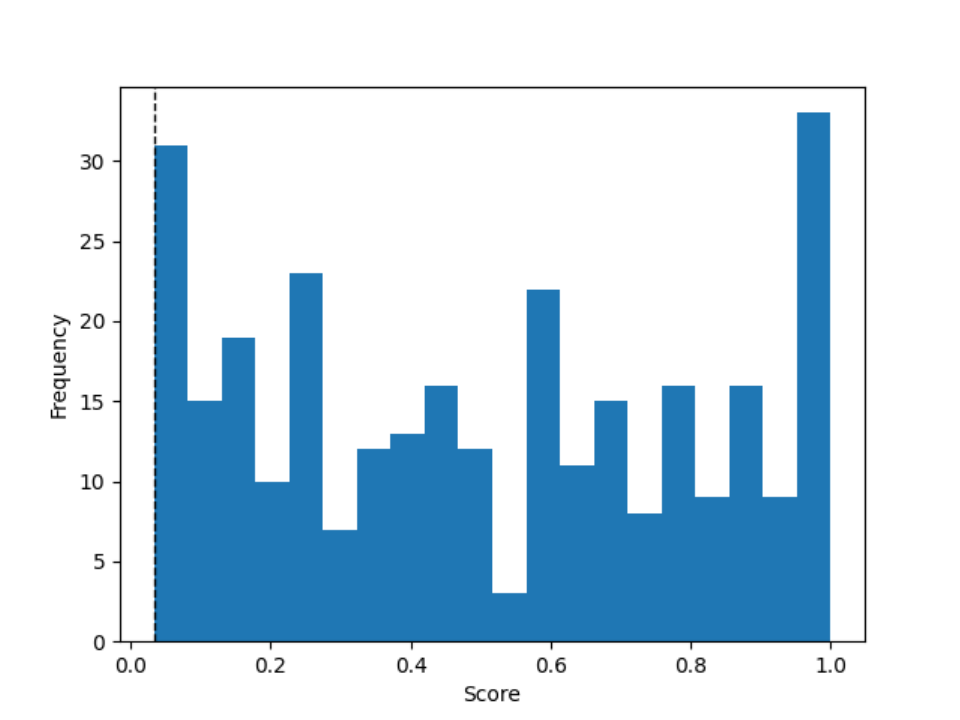}
\caption{True Scores on TriviaQA using Llama-2} \label{fig:true_a}
\end{subfigure}\hspace*{\fill}
\begin{subfigure}{0.45\textwidth}
\includegraphics[width=\linewidth]{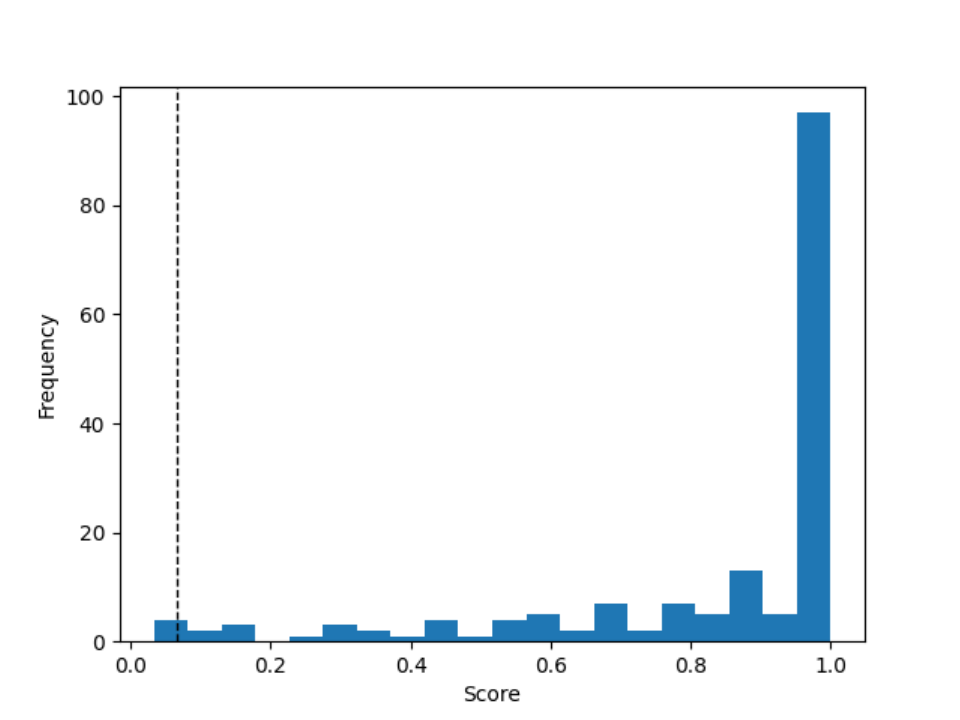}
\caption{True Scores on Natural Question using Llama-2} \label{fig:true_b}
\end{subfigure}
\caption{True scores collected on TriviaQA and Natural Question using Llama-2.} \label{fig:true_scores}
\end{figure}


\begin{table}
\centering
\begin{tabular}{llllll}
\toprule
 Task & Cov(\%)  & TRAQ & Bonf & TRAQ-P & Bonf-P \\
\midrule
\multirow[t]{5}{*}{NQ} & 50 & $\mathbf{4.8_{0.7}}$ & $5.0_{0.7}$ & $6.5_{0.9}$ & $6.6_{0.9}$ \\
 & 60 & $\mathbf{6.1_{1.0}}$ & $\mathbf{6.1_{1.0}}$ & $8.3_{1.2}$ & $8.5_{0.9}$ \\
 & 70 & $8.0_{0.9}$ & $\mathbf{7.9_{1.0}}$ & $10.6_{1.2}$ & $10.7_{1.2}$ \\
 & 80 & $\mathbf{10.6_{1.1}}$ & $10.7_{1.2}$ & $13.5_{1.8}$ & $14.7_{1.3}$ \\
 & 90 & $\mathbf{14.2_{1.9}}$ & $15.6_{1.7}$ & $21.5_{1.6}$ & $25.0_{6.5}$ \\
\cline{1-6}
\multirow[t]{5}{*}{Trivia} & 50 & $\mathbf{4.3_{0.5}}$ & $4.5_{1.2}$ & $5.7_{1.4}$ & $6.5_{1.2}$ \\
 & 60 & $\mathbf{5.8_{1.1}}$ & $6.8_{1.2}$ & $7.7_{1.4}$ & $9.2_{2.2}$ \\
 & 70 & $\mathbf{8.6_{1.6}}$ & $10.2_{1.9}$ & $13.4_{2.0}$ & $18.5_{6.2}$ \\
 & 80 & $\mathbf{15.1_{2.1}}$ & $19.1_{6.2}$ & $29.3_{2.6}$ & $71.3_{60.7}$ \\
 & 90 & $\mathbf{117.9_{51.2}}$ & $122.6_{53.4}$ & $145.2_{8.0}$ & $157.2_{7.7}$ \\
\cline{1-6}
\multirow[t]{5}{*}{SQuAD1} & 50 & $\mathbf{4.5_{0.4}}$ & $5.1_{0.4}$ & $5.2_{0.6}$ & $6.4_{0.5}$ \\
 & 60 & $\mathbf{5.7_{0.4}}$ & $6.5_{0.6}$ & $6.9_{0.8}$ & $7.7_{0.7}$ \\
 & 70 & $\mathbf{7.6_{0.6}}$ & $8.1_{0.9}$ & $8.6_{0.6}$ & $10.2_{0.6}$ \\
 & 80 & $\mathbf{9.5_{0.7}}$ & $11.4_{1.1}$ & $11.8_{1.2}$ & $14.4_{2.0}$ \\
 & 90 & $\mathbf{15.1_{1.9}}$ & $18.0_{2.0}$ & $21.9_{2.7}$ & $23.7_{2.2}$ \\

\cline{1-6}
\bottomrule
\end{tabular}

\caption{Average semantic counts using Llama-2.}
\label{tab:all_semantic_llama}
\end{table}

\subsection{Additional Qualitative Results}
\label{sec:qual}
\subsubsection{All Covered}
As shown in the example below, when the first retrieved passage is sufficiently informative, the LLM can probably generate correct responses for the question. In this case, TRAQ and Bonf can also include semantically correct responses in the aggregated sets. Again, TRAQ included less semantic meanings than Bonf did.






\begin{lstlisting}
Query: who plays zack and cody in the suite life

True answer: ['Dylan and Cole Sprouse']

Standard: {'Dylan and Cole Sprouse', 'Dylan and Cole Sprouse.'}

TRAQ: {'Dylan and Cole Sprouse', 'Dylan Sprouse', 'Phill Lewis'}

Bonf: {'Dylan and Cole Sprouse', 'Cole Sprouse',  'Dylan Sprouse'}
\end{lstlisting}

\subsection{Miscovered}
If the first retrieved passage lacks information, the standard RAG pipeline may struggle to provide the correct answer. However, in such scenarios, TRAQ and Bonf can construct prediction sets that contain the correct response with high probability, with TRAQ constructing smaller prediction sets.

\begin{lstlisting}
Query: who sang i love rock and roll original

True Answer: ['Alan Merrill']

Standard: {'Joan Jett'}

TRAQ: {'Joan Jett', 'Elvis Presley', 'Lou Reed',  'Joan Jett \& the Blackhearts',  'Alan Merrill', 'Chuck Berry', 'Donna Summer', 'Kevin Johnson', 'Joan Jett and The Arrows'}

Bonf: {'Joan Jett', 'Elvis Presley', 'The Velvet Underground', 'Lou Reed', 'Joan Jett & the Blackhearts', 'Alan Merrill',  'Chuck Berry', 'Donna Summer', 'Bobby Vee', 'Buddy Holly', 'Kevin Johnson', 'Mac Davis', 'The original version of "I Love Rock and Roll" was sung by The Arrows.', 'The Runaways', 'The answer to the question is not provided in the given context.', 'The Runaways sang the original version of "I Love Rock and Roll".', 'Joan Jett and The Arrows'}    
\end{lstlisting}

\section{Implementation Details}
\subsection{Llama-2 Fine-tune Hyperparameters}
\label{sec:llama_ft}
We use 4-bit QLoRA~\cite{dettmers2023qlora} to fine-tune the Llama-2~\cite{touvron2023llama2} models on Natural Question, TriviaQA, and SQuAD-1 datasets separately. The hyperparameters used for QLoRA are listed in Table~\ref{tab: qlora}; and the fine-tuning parameters are listed in Table~\ref{tab:ft}.

\begin{table}[H]
\centering
\begin{tabular}{|ll|ll|} \toprule
Name       & Value  & Name     & Value   \\ \midrule
r       & 64  & alpha     & 16   \\
dropout & 0.1 & precision & 4bit \\ \bottomrule
\end{tabular}
\caption{QLoRA hyperparameters.}
\label{tab: qlora}
\end{table}

\begin{table}[H]
\centering
\begin{tabular}{|ll|ll|} \toprule
Name       & Value  & Name     & Value   \\ \midrule
batch\_size   & 16    & learning rate & 2e-4     \\
weight\_decay & 0.001 & lr scheduler  & constant \\
warmup ratio  & 0.03  & epoch         & 3   \\ \bottomrule 
\end{tabular}
\caption{Fine-tuning hyperparameters.}
\label{tab:ft}
\end{table} 

\subsection{Fine-tune Dense Passage Retriever (DPR) on the Biomedical Dataset (BioASQ)}
\label{sec:ft_bio}
We collect our dataset for DPR fine-tuning by using the collection of all the passages mentioned in BioASQ as our knowledge corpus, resulting in \textbf{56,795} passages. Following the method in ~\cite{karpukhin-etal-2020-dense}, we create negative contexts for each sample in BioASQ by first retrieving the \textbf{top-20} passages; and labeling contexts that did not contain the golden answers as the \textbf{negative passages}. We then divide the original BioASQ dataset into training, validation, and testing sets, with 3,775, 471, and 469 data points, respectively.

We fine-tune the DPR model~\cite{karpukhin-etal-2020-dense} using the \textit{Haystack} framework~\cite{QuickSta63:online}, adjusting key hyperparameters to \textbf{epochs=5} and \textbf{batch size=16}. Other hyperparameters are left at their default values. To evaluate the performance of the fine-tuned DPR, we use \textit{hit rate}, which is the rate of relevant passages included in the top k retrieved passages. With $k$ set to \textbf{20}, the fine-tuned DPR achieves hit rates of \textbf{77.2\%} on the training set, \textbf{72.8\%} on the validation set, and \textbf{75.7\%} on the testing set.

\subsection{Different Prompts}
\label{sec: prompts}

\begin{tcolorbox}[enhanced,title=Zero-shot Prompt,
fonttitle=\bfseries,coltitle=green!25!black,
attach boxed title to top center=
{yshift=-2mm,yshifttext=-1mm},
boxed title style={colframe=green!75!black,
colback=yellow!50!green}]

Answer the following question based on the given context; Answer the question shortly.\\
\\
Question: \{question\} \\
Context: \{context\} \\
Answer:
\label{zero}
\end{tcolorbox}

\begin{tcolorbox}[enhanced,title=Few-shot Prompt,
fonttitle=\bfseries,coltitle=green!25!black,
attach boxed title to top center=
{yshift=-2mm,yshifttext=-1mm},
boxed title style={colframe=green!75!black,
colback=yellow!50!green}]

Answer the following question based on the given context; Answer the question shortly.\\
\\
Question: \{question 1\} \\
Context: \{context 1\} \\
Answer: \{answer 1\} \\
\\
Question: \{question 2\} \\
Context: \{context 2\} \\
Answer: \{answer 2\} \\
\\
Question: \{question\} \\
Context: \{context\} \\
Answer: 
\label{fewshot}
\end{tcolorbox}

\begin{lstlisting}
['The Great Lakes do not meet the ocean.',
'The Great Lakes meet the ocean at the Saint Lawrence River.',
'The Great Lakes meet the ocean through the Saint Lawrence River.',
'The Great Lakes do not meet the ocean.',
'The Great Lakes do not directly meet the ocean.',]

['There is no specific answer given in the provided context about where the Great Lakes meet the ocean.',
'Atlantic Ocean',
'Saint Lawrence River',
'The Great Lakes do not meet the ocean.',
'The Great Lakes do not meet the ocean. They are primarily connected to the Atlantic Ocean through the Saint Lawrence River.',
'The Great Lakes do not meet the ocean. They connect to the Atlantic Ocean through the Saint Lawrence River.',
'The Great Lakes meet the ocean through the Saint Lawrence River.',
'They do not meet the ocean.']
\end{lstlisting}

\begin{figure}[htbp!] 
\begin{subfigure}{0.32\textwidth}
\includegraphics[width=\linewidth]{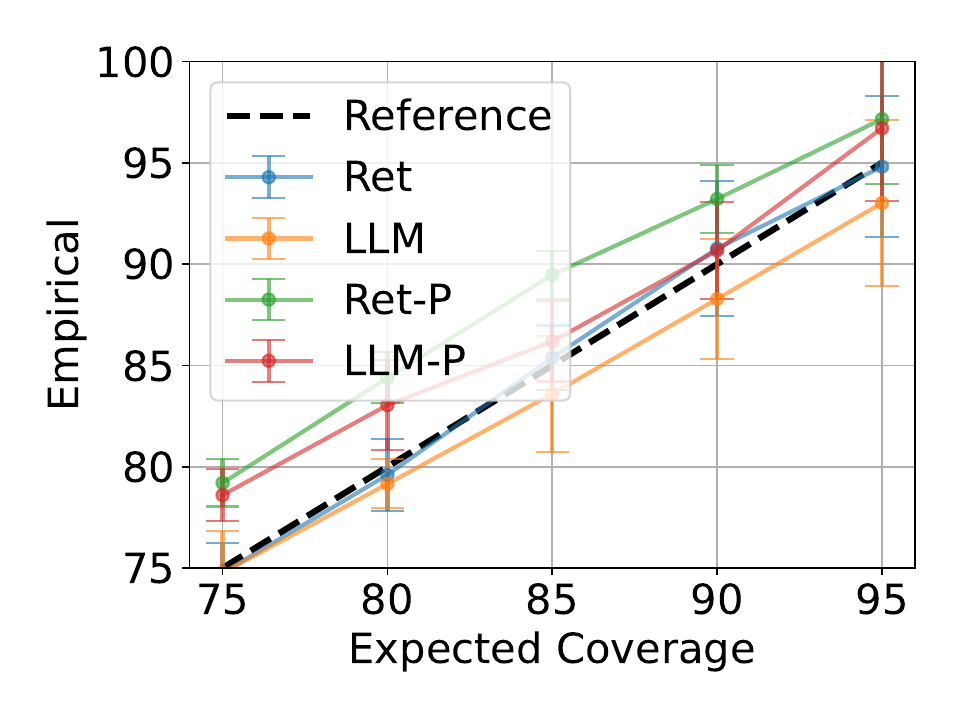}
\caption{End-to-end guarantee considering only the most relevant passage} \label{fig:semantic_a}
\end{subfigure}\hspace*{\fill}
\begin{subfigure}{0.32\textwidth}
\includegraphics[width=\linewidth]{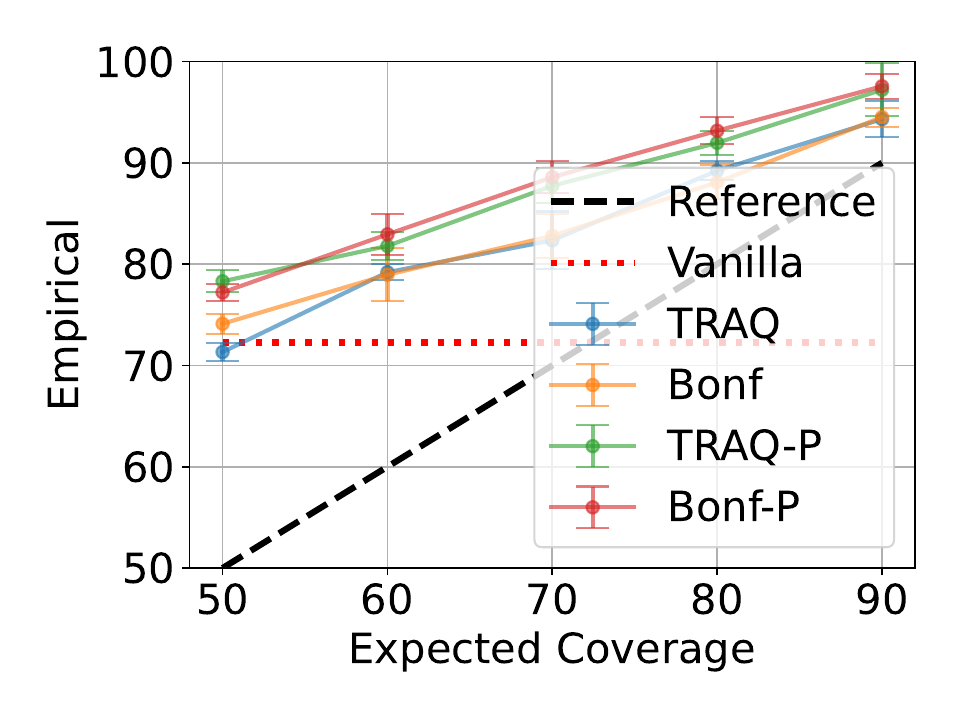}
\caption{Overall coverage guarantee considering all passages} \label{fig:semantic_b}
\end{subfigure}\hspace*{\fill}
\begin{subfigure}{0.32\textwidth}
\includegraphics[width=\linewidth]{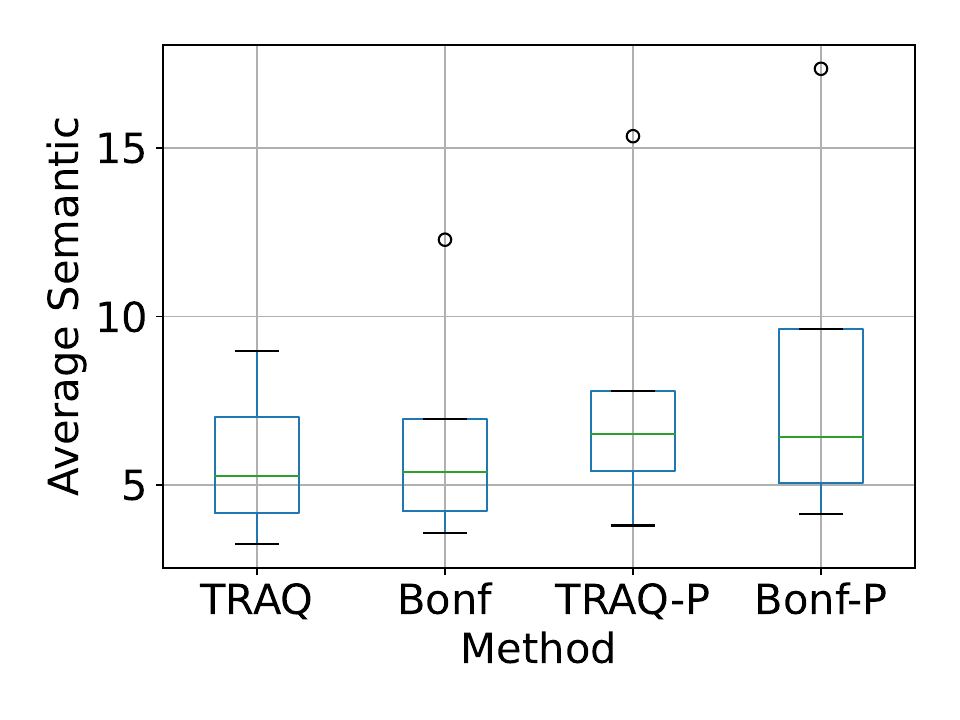}
\caption{Average prediction set sizes\\ \quad} \label{fig:semantic_c}
\end{subfigure}

\caption{Results using a few-shot prompt on Natural Question using GPT-3.5.} \label{fig:additional_semantic}
\end{figure}

\subsection{Main Packages}
\begin{table}[H]
\centering
\begin{tabular}{|ll|ll|}
\toprule
Package & Version & Package  & Version \\
\midrule
transformer~\cite{wolf2020huggingfaces} & 4.32.1 & nltk~\cite{bird2009natural}  & 3.8.1 \\
spacy~\cite{spacy2}       & 3.6.1  & torch~\cite{NEURIPS2019_9015} & 2.0.1 \\
rouge-score~\cite{lin-2004-rouge} & 0.1.2  & scikit-optimize~\cite{head_2021_5565057}      &  0.9.0     \\ \bottomrule
\end{tabular}
\end{table}

\subsection{Artifact License and Terms}
Our implementation is based on \textit{haystack}, \textit{transformers} and \textit{DPR}~\cite{karpukhin-etal-2020-dense}. The first two are licensed under \textbf{Apache License 2.0}, the third is licensed under \textbf{Attribution-NonCommercial 4.0 International}. We used four datasets, namely BioASQ, Natural Question, TriviaQA, and SQuAD-1. BioASQ is licensed under the \textbf{CC BY 2.5 license}, Natural Question is under \textbf{CC BY-SA 3.0 license}, TriviaQA is under the \textbf{Apache License 2.0}, and SQuAD-1 is under the \textbf{CC BY-SA 4.0 license}. We used two LLMs, namely \textit{GPT-3.5} and \textit{Llama-2}. GPT-3.5 usage is subject to OpenAI's \textit{Sharing \& Publication Policy} and \textit{Usage Policies}. Llama-2 is licensed under the Llama-2 Community License~\cite{Llamaacc89:online}. Our implementation and the data collected are under the \textbf{MIT License}.

Our use of the existing artifacts is consistent with their original intended use. Our created artifacts intend to verify our proposed method in our submission, which is consistent with original access conditions. 

\section{Removing Assumption~\ref{as:chat}}
\label{sec: idk}

In certain scenarios, even if the most pertinent passage is identified and given to the language understanding model (LLM), the LLM is still unable to answer the question with accurate answers. This could be due to a variety of reasons, such as the passage not being sufficiently specific or the LLM not being able to extract enough information from the passage. If the LLM is unable to generate correct responses even when the most pertinent passage is provided, our guarantee regarding the LLM and end-to-end pipeline may not hold. This problem can be alleviated by annotating better passages or using more powerful LLMs.

To address the issue with existing datasets and language models, we offer the guarantee of claiming \textit{I do not know} if the language model is unable to generate a correct response to a question and its most relevant passage. We collect questions and their most relevant passages, and also labels that indicate whether GPT-3.5 could generate a correct response. We then divided the dataset into training, validation, and testing sets, with 6,899, 1,725, and 1,725 data points, respectively. We label \textbf{True} if the language model could generate a correct response and \textbf{False} otherwise. We then train a BERT-based text classifier, which takes in the questions and their most relevant passages, and predicts whether GPT-3.5 can generate a correct response. We name the trained classifier \textit{Conf-Classifier}. Surprisingly, the Conf-Classifier achieves an accuracy of 95\% on the testing set. To provide guarantees, we apply conformal prediction to the outputs of the Conf-Classifier. We include \textit{I do not know} in the LLM set if the constructed prediction set contained \textbf{False}. 

To construct the calibration set, we collect estimated confidences on \textbf{not being able to answer the question} on input questions in which the LLM fails to generate the correct response. We denote these estimated confidences as $\{s_1, \ldots, s_N\}$. Given a user-specified coverage level, we then use conformal prediction to identify the $\frac{\lceil(N+1)(1-\alpha)\rceil}{N}$ quantile as the threshold $\tau_{\text{Ign}}$ to construct the set. Given an input question $q$, we then include \textit{I do not know} in the aggregation set $C_{Agg}(q)$ if the estimated confidence $n_{K+1}$ is above $\tau_{\text{Ign}}$. Then we can guarantee the following:
\begin{lemma}
Given an input question $q$ that the LLM cannot correctly answer and a user-specified error level $\alpha$, if $\alpha_{\text{Ign}}$ is used to decide whether to include \textit{I do not know}, the aggregation set satisfies the following property:
\begin{equation*}
\Pr_{q \sim \mathcal{D}}[\text{I do not know} \in C_{\text{Agg}}(q)]
\end{equation*}
\end{lemma}
This result follows straightforwardly from Theorem~\ref{as:iid}.

We validate our guarantee using five distinct random seeds and five different coverage levels. The results are shown in Figure~\ref{fig:additional_idk}. As the figure illustrates, our method can include \textit{I do not know} at various required coverage levels. By combining this with our guarantee on the LLM, we can guarantee all questions. 

\begin{theorem}
Given a user-specified error level $\alpha$, if aggregation is constructed with error level $\alpha$, the resulting prediction sets contain true answers (i.e. semantically correct responses if the input question is answerable; or \textit{I do not know} if the input question is unanswerable) with probability at least $1-\alpha$, i.e. 
\begin{equation*}
\Pr_{q \sim \mathcal{D}}[\text{True answer} \in C_{\text{Agg}}(q)] \geq 1-\alpha.
\end{equation*}
\label{th:all}
\end{theorem}

\begin{proof}
Suppose we construct the aggregation set and ignorance set both with coverage level $1-alpha$; then we have the following inequalities:
\begin{align*}
&\Pr_{q \sim \mathcal{D}}[\text{True answer in the resulting set}] \\
&= \Pr_{q \sim \mathcal{D}}[\text{Correct response} \in C_{\text{Agg}}(q)] \times \Pr[\text{q is answerable}] \\ &+ \Pr_{q \sim \mathcal{D}}[\text{I do not know} \in C_{\text{Agg}}(q)] \times \Pr[\text{q is unanswerable}]\\
&\leq (1-\alpha) \times \Pr[\text{q is answerable}] + (1-\alpha) \times \Pr[\text{q is unanswerable}] \\
&= 1-\alpha.
\end{align*}
\end{proof}

\begin{figure}[t!] 
\begin{subfigure}[t]{0.32\textwidth}
\includegraphics[width=\linewidth]{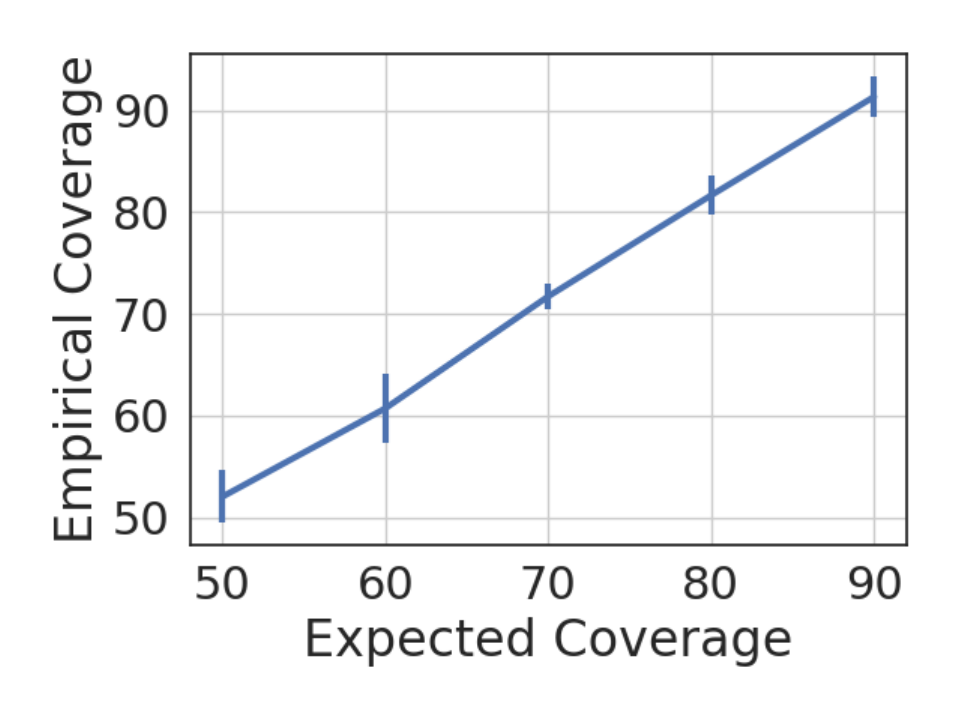}
\caption{Coverage Rate on \textit{I do not know}.} \label{fig:idk_a}
\end{subfigure}\hspace*{\fill}
\begin{subfigure}[t]{0.32\textwidth}
\includegraphics[width=\linewidth]{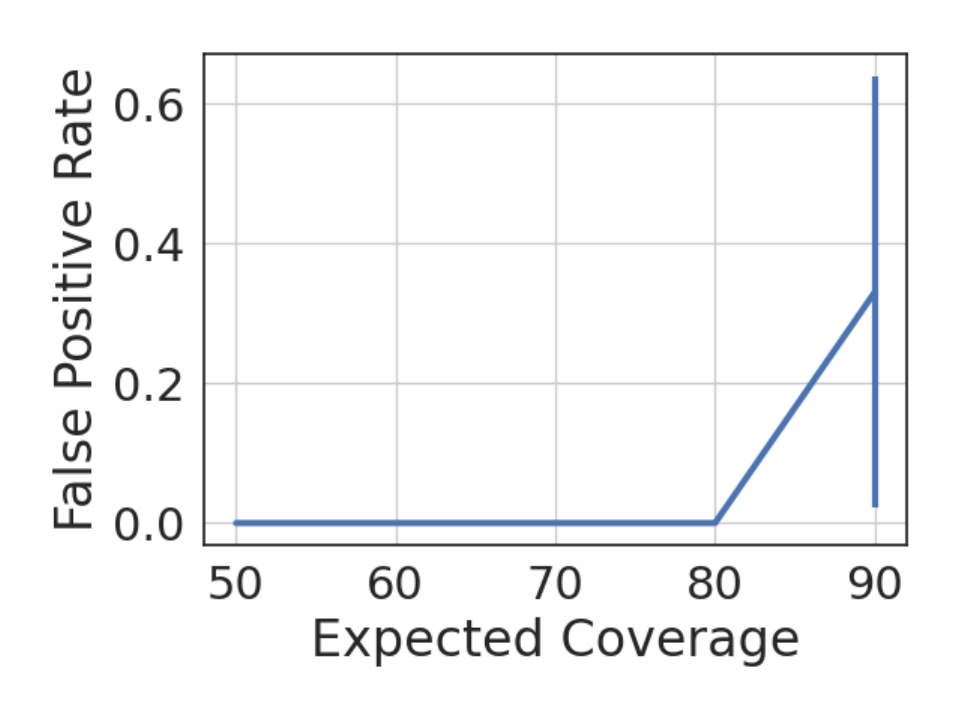}
\caption{False Positive Rates (claiming \textit{I do not know but actually being able to answer}.} \label{fig:idk_b}
\end{subfigure}\hspace*{\fill}
\begin{subfigure}[t]{0.32\textwidth}
\includegraphics[width=\linewidth]{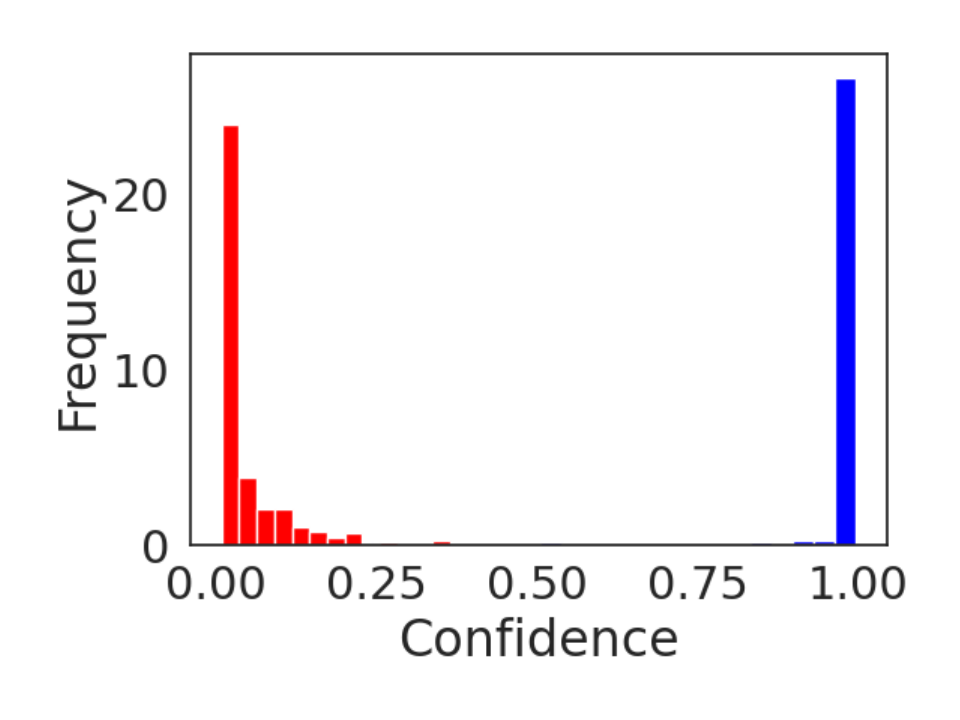}
\caption{The distribution of confidence on claiming \textit{I do not know} using the training classifier.} \label{fig:idk_c}
\end{subfigure}

\caption{Results on identifying whether a given prompt is answerable or not.} \label{fig:additional_idk}
\end{figure}

\section{AI Assistant Usage}
We used \textit{Copilot} to assist our coding.

\end{document}